\pdfoutput=1
\documentclass{article} %
\PassOptionsToPackage{table,xcdraw}{xcolor}
\usepackage{preprint,times}
\prepfinalcopy

\usepackage[utf8]{inputenc}
\usepackage[T1]{fontenc}

\usepackage[most]{tcolorbox}
\usepackage{hyperref}
\usepackage{url}
\usepackage{booktabs}
\usepackage{nicefrac}
\usepackage{microtype}
\usepackage[table,xcdraw]{xcolor}
\usepackage{xspace}

\usepackage{amsmath}
\usepackage{amssymb}
\usepackage{amsfonts}
\usepackage{mathtools}
\usepackage{amsthm}
\usepackage{bbm}
\usepackage{thmtools}
\usepackage{enumitem}

\usepackage{algorithm}
\usepackage[noend]{algpseudocode}

\usepackage{graphicx}
\usepackage{float}
\usepackage{caption}
\usepackage{subcaption}
\usepackage{wrapfig}
\usepackage{multicol}
\usepackage{multirow}
\usepackage{longtable}

\usepackage{cleveref}

\usepackage{tikz}
\usetikzlibrary{
  positioning,
  calc,
  shapes.geometric,
  shapes,
  shapes.multipart,
  arrows.meta,
  arrows,
  decorations.markings,
  external,
  trees,
  backgrounds,
  automata,
  shapes.misc,
  fit
}
\usepackage{scalefnt}

\tikzset{
  -Latex,
  auto,
  node distance=1 cm and 1 cm,
  semithick,
  state/.style={ellipse, draw, minimum width=0.7 cm},
  point/.style={circle, draw, inner sep=0.04cm, fill, node contents={}},
  nnh/.style={
    rectangle,
    draw,
    thick,
    minimum width=1.5cm,
    minimum height=1.0cm
  },
  nnv/.style={
    rectangle,
    draw,
    very thick,
    fill=gray!28,
    inner sep=0.04cm,
    minimum width=1.2cm,
    minimum height=1.2cm,
    rounded corners=0.05cm
  },
  nnvsm/.style={
    rectangle,
    draw,
    very thick,
    fill=gray!28,
    inner sep=0.0cm,
    minimum width=1.0cm,
    minimum height=1.0cm,
    rounded corners=0.05cm
  },
  outer/.style={inner sep=3pt, fill=blue!15},
  outer1/.style={inner sep=3pt, fill=green!15},
  outer1t/.style={inner sep=0pt, fill=green!15},
  louter/.style={inner sep=5pt, fill=blue!15},
  XOR/.style={
    draw,
    circle,
    append after command={
      [shorten >=\pgflinewidth, shorten <=\pgflinewidth]
      (\tikzlastnode.north) edge (\tikzlastnode.south)
      (\tikzlastnode.east) edge (\tikzlastnode.west)
    }
  },
  bidirected/.style={Latex-Latex, dashed},
  el/.style={inner sep=2pt, align=left, sloped},
  cross/.style={
    cross out,
    draw=black,
    minimum size=2*(#1-\pgflinewidth),
    inner sep=0pt,
    outer sep=0pt
  },
  cross/.default={1pt}
}

\newtheorem{theorem}{Theorem}
\newtheorem{lemma}{Lemma}
\newtheorem{proposition}{Proposition}
\newtheorem{definition}{Definition}
\newtheorem{assumption}{Assumption}
\newtheorem{corollary}{Corollary}

\newcommand{\V}{\mathcal{V}}

\newcommand{\mbf}{\mathbf}

\newcommand{\G}{\mathbb{G}}

\newcommand{\Do}{\text{do}}

\newcommand{\entails}{\vDash}

\newcommand{\U}{\overline{U}}
\newcommand{\Up}{\underline{U}}
\newcommand{\Lp}{\overline{L}}
\newcommand{\Lo}{\underline{L}}

\newcommand{\Ux}{\overline{U_{x}}}
\newcommand{\Upx}{\underline{U_{x}}}
\newcommand{\Lpx}{\overline{L_{x}}}
\newcommand{\Lox}{\underline{L_{x}}}

\newcommand{\Uate}{\overline{U}}
\newcommand{\Upate}{\underline{U}}
\newcommand{\Lpate}{\overline{L}}
\newcommand{\Late}{\underline{L}}

\newcommand{\Uatex}{\overline{U_{x}}}
\newcommand{\Upatex}{\underline{U_{x}}}
\newcommand{\Lpatex}{\overline{L_{x}}}
\newcommand{\Latex}{\underline{L_{x}}}

\newcommand{\CR}{\mathcal{C}}
\usepackage{xspace}
\newcommand{\myalgo}{\texttt{UA-DCM}\xspace}

\newcommand{\epSample}{sample }
\newcommand{\epNonID}{nonID }
\newcommand{\four}{[\Lo_x, \Lp_x, \Up_x, \U_x]}
\newcommand{\inner}{[\Lp_x, \Up_x]}
\newcommand{\outr}{[\Lo_x,\Lp_x]\cup[\Up_x,\U_x]}

\DeclarePairedDelimiter{\ceil}{\lceil}{\rceil}

\usepackage{subfiles}

\title{Uncertainty-aware Causal Decision Making \\ via  Effect Bound Decomposition}

\author{%
  Md Musfiqur Rahman$^{1,*}$ \quad
  Ziwei Jiang$^{2,*}$ \quad
  Hilaf Hasson$^{3,\dagger}$ \quad
  Murat Kocaoglu$^{2}$ \\
  {\small $^{1}$Electrical and Computer Engineering, Purdue University} \\
  {\small $^{2}$Computer Science, Johns Hopkins University} \\
  {\small $^{3}$Cohesity} \\
  {\small \texttt{rahman89@purdue.edu, zjiang85@jh.edu,}} \\
  {\small \texttt{hilaf.hasson@cohesity.com, mkocaoglu@jhu.edu}}
}
\begin{document}

\maketitle
{\renewcommand{\thefootnote}{\fnsymbol{footnote}}\footnotetext[1]{Equal contribution.}\footnotetext[2]{Work completed while at Intuit AI Research.}}

\begin{abstract}
Causal inference from observational data can provide strong evidence for finding the best action in a decision-making scenario without having to perform expensive randomized trials. The causal effect of an action is often not pointwise identifiable even with infinite data due to unobserved confounding factors. Furthermore, having only finitely many samples adds another layer of uncertainty to causal effect estimation. Several existing methods can be used to obtain upper and lower bounds to the causal effect, ranging from symbolic methods to the more recent neural network-based approaches, which implicitly incorporate both sources of uncertainty. However, these methods do not inform whether collecting more samples may or may not help identify the best action from observational data, leaving experts in the dark about their data collection strategies. We address this problem with a novel framework that can distinguish the range of causal effect values that might be eliminated by collecting more samples from the range of values that, with high probability, cannot be eliminated with more observational samples. We show that this partitioning can be obtained by solving max-min and min-max optimization problems. We leverage neural causal models to approximately recover this decomposition in practice. We demonstrate via experiments on synthetic and real-world datasets that our algorithm can determine when collecting more samples will not help determine the best action. Our framework can help practitioners decide when to resort to non-observational studies or seek to measure some of the unmeasured confounders for optimal decision-making. 

\end{abstract}

\section{Introduction }
Causal decision-making aims to identify the best action relative to a goal. For example, \emph{which drug maximizes the life expectancy of an average patient? Which candidate government policy should be enacted to reduce unemployment more?} Since actions affect the system relative to its observational state, answering such questions requires causal models. 
Pearl's structural causal models (SCMs) provide us with a systematic set of complete algorithmic principles to identify the effect of actions, or post-interventional distributions from observational data  %
under well-defined assumptions ~\citep{pearl1995causal,tian2002studies, huang2006identifiability,shpitser2008complete}. %
When the interventional distribution or the causal effect cannot be determined uniquely, the framework can be used to obtain bounds on these quantities%
~\citep{zhang2021bounding,hu2021generative,li2022bounds}. %

For decision-making, we may wish to understand if an action, say enforcing $X=x$, increases the probability of a certain event, say $Y=1$. In SCM language, we want to test if $P(Y=1|\Do(x))>P(Y=1)$. Or we may want ot decide which of the two actions, enforcing $X=x_0$ or enforcing $X=x_1$, leads to a higher probability of the said event, which we would test by asking if $P(Y=1|\Do(x_1))>P(Y=1|\Do(x_0))$. Equivalently, we can directly test whether the average treatment effect (ATE), defined as $\mathbb{E}[Y|\Do(x_1)]-\mathbb{E}[Y|\Do(x_0)]$, is positive. These can be extended to having more than two actions, for example, by testing if $P(Y=1|\Do(x_0))>\max\limits_{i\neq 0}P(Y=1|\Do(x_i))$.
However, these interventional quantities that are essential for decision-making are not always identifiable even with infinite data, i.e., there may be different SCMs that entail the observational distribution, but lead to different interventional quantities. %
On top of this, %
in any realistic scenario, we have a \textit{finite} dataset from the observational distribution, 
which introduces another source of  uncertainty for estimating these interventional quantities. %

\begin{figure}
    \vspace{-5mm}
    \centering
    \includegraphics[width=1\linewidth]{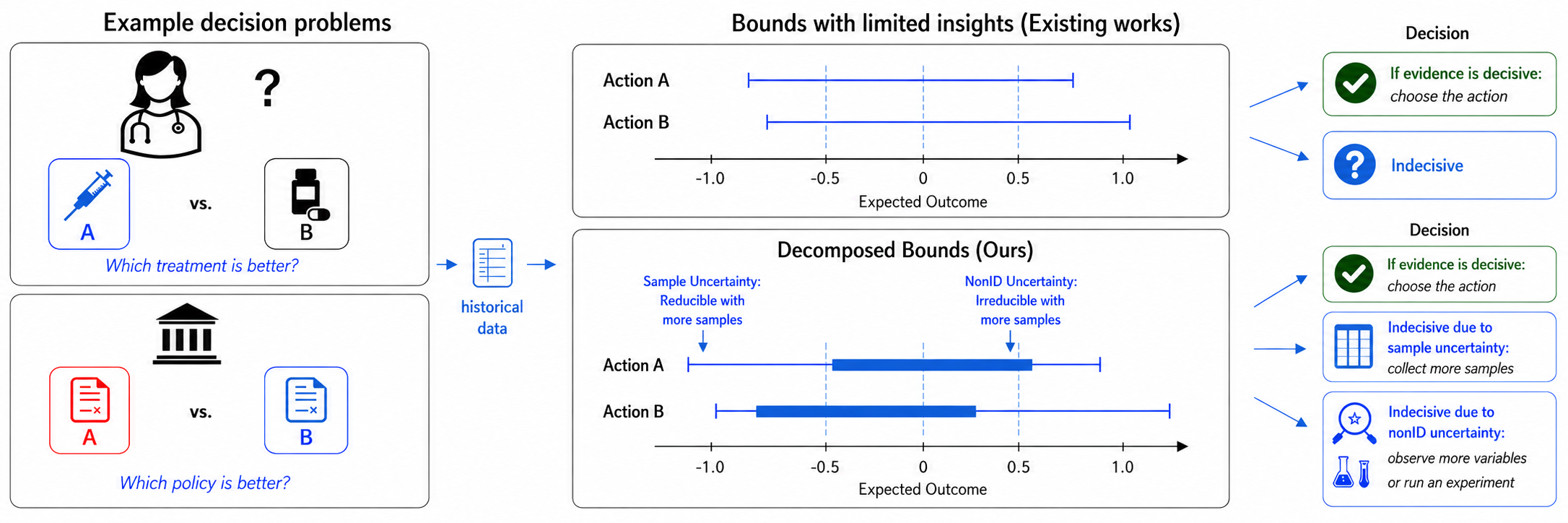}
    \caption{Decision making with decomposed causal effect bounds.}
    \label{fig:placeholder}
    \vspace{-2mm}
\end{figure}

If, despite these sources of uncertainty, the \emph{lower bound} to the utility of Action $1$ is greater than the \emph{upper bound} to the utility of all other actions, we can still decide Action 1 is more desirable than the others. However, often, compounding these two sources of uncertainty leads to a wide range of values for the interventional quantities of interest which leaves optimal action ambiguous. Importantly then it is not clear whether bounds were wide due to not having sufficiently many observational samples, or due to structural reasons that even with infinite samples one would not be able to make a decision. In the special cases where interventional quantites are identifiable due to the causal graph, we know collecting more samples will allow decision-making since the range of possible interventional values will converge to a single point (either ATE is positive or negative, or $P(Y=1|\Do(x_1))$ is either greater or less than $P(Y=1|\Do(x_0))$). However, to the best of our knowledge, it is currently unknown how to answer this question for non-identifiable causal queries. 

We address this problem through a novel decomposition of the range of interventional values into two subsets: A range of interventional values that will not be removed with high probability by increasing the number of samples in our dataset, called the nonID region, and the rest that may be reduced by collecting more samples (e.g., by waiting for more data to be collected via sensors in an environment).   
In a causal decision-making scenario, this nonID region crucially helps us to decide the next move: collect more samples from observed distribution, or resort to a different distribution, e.g., by adding new sensors to observe more variables, or resorting to a randomized trial. We show that the proposed decomposition can be computed by solving %
min-max and max-min problems. Finally, to approximate these bounds in practice %
we develop an approach that leverages %
neural/deep causal generative models~\citep{kocaoglu2018causalgan,balazadeh2022partial,rahman2024modular,xia2021causal}. Our contributions are summarized below:
\begin{itemize}
\item We propose a novel decomposition of the range of interventional quantities of interest into structural and sample uncertainty. This is, to the best of our knowledge, the first time a method can infer with high probability that collecting more observational samples can not help with finding the optimal action for non-identifiable causal queries. 

\item We establish a novel algorithm, \myalgo, that approximately recovers our \emph{nonID uncertainty} and \emph{sample uncertainty} in practice. 
 
\item
We demonstrate the utility of our algorithm through extensive experiments on synthetic data of various complexities and on a real-world parent labor dataset.
\end{itemize}

\section{Background}

We formalize the observational data-generating process with a structural causal model.

\begin{definition}[Structural causal model, SCM)~\citep{pearl2009causality}]
 \normalfont
An SCM $S$ is a tuple of five elements:
{$ S=(\mathbf{V}, \mathcal{N}, \mathcal{U}, \mathcal{F}, P(.) )$}.  
Here, each observed variable $V_i\in\mathbf{V}$ is realized as an evaluation of the function $f_i\in\mathcal{F}$ by taking as input a subset of the remaining observed variables $Pa_i\subset \mathbf{V}$, an exogenous noise variable $E_i\in \mathcal{N}$, and optionally an unobserved confounding variable $U_i\in\mathcal{U}$. 
$P(.)$ is a {product joint} over all unobserved $\mathcal{N}\cup\mathcal{U}$. 
An SCM containing unobserved confounders is called a \textbf{Non-Markovian causal model}.
A \emph{causal graph}, $G$, representing the variables as nodes and structural functional relationships encoded in the SCM as edges, is called an Acyclic Directed Mixed Graph (ADMG) and denoted as $S \models G$.

\end{definition}

The best action is chosen based on the causal effect of each action on the outcome variable $Y$, which we formalize using do-interventions.
\begin{definition}[Causal effect and do-intervention]
 \normalfont
A do-intervention ${do(}x)$ replaces the functional equation $X=f_X$ with a specific value $X=x$ without affecting other equations. The distribution induced on the variable $Y$ after such an intervention is called an interventional distribution  $P(Y|\Do(x))$. With no intervention, the observational joint distribution of $\mathbf{V}$ is $P(\mathbf{V})$.
\end{definition}

To efficiently find the bounds for causal effects, we model the SCM with deep neural networks.
\begin{definition}[Deep causal models, DCM)~\citep{kocaoglu2018causalgan,xia2021causal,rahman2024modular}]
\label{def:scm}
 \normalfont
	A neural net architecture $\mathbb{G}$ is called a deep causal model (DCM) for an ADMG $G=(\mathbf{V},\mathcal{E})$ if it consists of a collection of neural nets,  each $f_i$ accepts a sufficiently high-dimensional noise vector $N_i$, 
		ii) the output of $f_j$ is input to $f_i$ iff $V_j\in Pa_G(V_i)$,
		iii) $N_i=N_j$ iff $V_i,V_j$ share an unobserved confounder. 
Functions in DCM, $\G=\{f_{1},...,f_{n}\}$ are parameterized by $\Theta= \{\theta_1, ... , \theta_{|\V|}\}$.
Similar to the original data distribution, $P(\V)$, we define ${P}_{\theta}(\V)$ to be the distribution induced by the DCM.
Sufficiently high-dimensional noise vectors $N_i$, e.g., Gaussian, can replace both the exogenous noises and the unobserved confounders in the true SCM. We consider DCM to be \emph{representative enough for an SCM} if the neural networks have sufficiently many parameters to induce the same observed distribution as the true SCM. 
When a latent confounder $U$ affects two observed variables $X$ and $Y$, we can match the joint distribution $P(X,Y)$ by feeding the same noise $N_X=N_Y$ (as latents) into both $f_X, f_Y$. 
\end{definition}
\begin{definition}[Causal effect with DCM]
\label{def:NCM-sample}
\normalfont
To perform a hard intervention ${do(}X=x)$, we manually set the values for the intervened variables as $X=x$ instead of using their neural network. Then, we feed forward those values into their children and the rest of the mechanisms will work as usual to generate their corresponding values from which we can estimate the causal effect $P(y|\Do(x))$ and ATE.
A loss function measuring $d(P(\mathbf{V}, P_{\theta}(\mathbf{V}))$ can be used to train DCM's end-to-end differentiable networks.  
When causal effects are not pointwise identifiable, we add an additional loss function to maximize or minimize the causal effects. Please check Theorem~\ref{th:identifiability} for details.
\end{definition}

\section{Related Work}
When the causal effect is not identifiable, %
one can identify a range of causal effects from the SCMs that are compatible with the observational data.
\citet{tian2000probabilities} used the response variable to get bounds of counterfactual queries from observational and interventional data. \citet{duarte2024automated} proposed an automated method to derive bounds for causal queries in arbitrary graphs.
Many researchers have explored the use of neural nets \citet{xia2021causal, balazadeh2022partial, rahman2024modular} to  design causal models and estimate causal queries by maximizing and minimizing the query under the semi-Markovian setting. 

In the uncertainty quantification literature,
\citet{melnychuk2024quantifying} proposed a method for quantifying aleatoric uncertainty in individualized treatment effects by deriving sharp bounds on the conditional distributions of the treatment effect. \citet{marmarelis2024ensembled} introduced an approach for constructing causal effect intervals with hidden confounding. %
Existing causal methods approach the decision-making problem from different perspectives.
\citet{frauen2023neural} proposed a neural framework for sensitivity analysis that is compatible with a large class of existing sensitivity models. \citet{jesson2020identifying} studied the decision-making problem under non-overlap. \citet{frauen2025treatment} propose an optimal decision-making approach with two-stage CATE estimators. 

These approaches cannot disentangle the uncertainty in causal effect bounds obtained from finite samples. Distributionally robust optimization (DRO)~\citep{bui2022unified,gao2023distributionally} also performs max-min or min-max optimization. However, while DRO searches for a model that is robust to distributions within an $\epsilon$-ball, we search for distributions within the ball that allow a model to achieve the best local minimum or the worst local maximum. Confidence intervals can also be added to existing causal effect estimators via bootstrapping, but neither these intervals nor existing causal effect bounds can distinguish the sources of uncertainty or suggest to a practitioner how to reduce it.

\citet{chernozhukov2013intersection} propose a precision-corrected estimator for intersection bounds. Although intersection bounds shares the same name as our paper, the object being intersected is different from that in our paper and serves a different purpose.
The intersection bounds correction \citep{chernozhukov2013intersection} is also applicable when the causal effect can be expressed with symbolic terms \citep{sachs2023general}. But the number of parameters grows exponentially with the number of variables in the DAG. And more importantly, existing work does not provide the decomposition of the uncertainty.
In contrast, our method
not only allows us to quantify the uncertainty of causal effects for queries that cannot be expressed as symbolic bounds, but also enables the decomposition of that uncertainty into finite-sample uncertainty and non-id uncertainty, which can improve downstream decision-making. 

\section{Causal Decision-Making with Finite Data}

We first define the average treatment effect that we will use to compare the effect of multiple actions.

\begin{definition}[Average Treatment Effect]
The average treatment effect (ATE) is defined for binary treatments as:
$  ATE = \mathbb{E}[Y|\Do(X=1)]- \mathbb{E}[Y|\Do(X=0)]$. For multiple actions, we define the ATE
of an action $X=x$ with respect to the best among remaining actions $x_b$ as:
$  ATE(x) = \mathbb{E}[Y|\Do(X=x)]- \max_{x_b\neq x}\mathbb{E}[Y|\Do(X=x_b)]$.
\end{definition}

\subsection{Interval for Identifiable and Non-identifiable ATE}
When randomized trials are not possible, we need an estimation of an interventional query from observational data.
The ID algorithm~\citep{tian2002general,shpitser2008complete} provides us with a complete characterization of causal queries that can be uniquely identified from observational joint distribution. However, this might be challenging in practice, as we do not have access to the true joint distribution from finite observational samples. 
We propose using a confidence set to quantify the uncertainty of the observational distribution. Specifically, for a probability distribution $P$, $\CR(\hat{P})$ is a random set such that $ \mathbb{P}( P\in \CR(\hat{P}))\geq 1-\alpha$, where $\mathbb{P}$ is the empirical measure induced by observing
\begin{wrapfigure}[12]{rh!}{0.4\textwidth}
\vspace{-1.2em}
\centering
\includegraphics[width=1\linewidth]{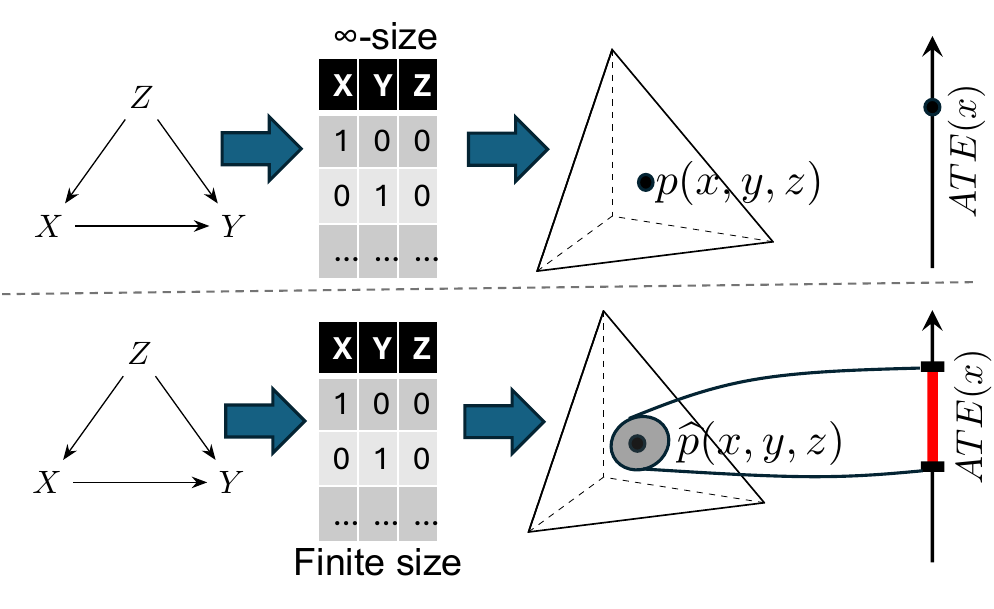}
\caption{ Given finite samples, we cannot pointwise estimate since the true distribution lies in a continuum of a confidence set, but only obtain an interval.}
\label{fig:bound-illustration}
\end{wrapfigure}
a given number of IID samples from $P$. 

Having a set of possible observational distributions converts a pointwise identifiable query to an 
interval for the causal effect. We establish this first: %

\begin{restatable}{theorem}{idSurj}
\label{thm:id-surj}
    Let $f(P,S)$ be the estimand of an identifiable causal effect for SCM $S$, with the observational distribution $P$. Let $\mathcal{P}$ be any connected set of observational distributions that contains $P$. %
    Let $a= \min_{p\in\mathcal{P}} f(P,S), b= \max_{P\in\mathcal{P}} f(P,S)$. Then $f|_{\mathcal{P}}:\mathcal{P}\rightarrow [a,b]$ is surjective. 
\end{restatable}

In words, the collection of ATEs %
for a set of observational distributions forms a contiguous interval in $[0,1]$. This establishes that it is sufficient to obtain maximum and minimum of ATE, to characterize this range of ATE values. For certain graphs %
this problem can be solved easily, e.g., backdoor graph, the objective can be reduced to a linear program (\Cref{sec:backdoor_exp}).

Next, we address non-identifiable queries under finite sample uncertainty. When the causal effect is identifiable, the only source of uncertainty is due to finite sample. With unobserved confounders, however, the effect may not be uniquely determined even with infinite samples: multiple SCMs consistent with the causal graph and the observational distribution may induce different causal effects. This is known as a non-identifiable query. For a fixed causal graph and observational distribution, existing algorithms can search the SCM space for the minimum and maximum causal effect, yielding a bound $P(y|\Do(x)) \in [L_{x}, U_{x}]$. 
Similar to Theorem~\ref{thm:id-surj}, under the conditions of Lemma~\ref{lm:connected}, the attainable values of a non-identifiable scalar causal query form an interval.
Therefore, its minimum and maximum are sufficient to characterize the entire identified set.

\begin{restatable}{theorem}{nonidsurj}
    \label{thm:nonid-surj}
    Let $f(p,S)$ be the estimand of some scalar non-identifiable causal effect of interest for a finite-state causal graph satisfying the conditions of Lemma~\ref{lm:connected}, and let $p$ be a compatible observational distribution.
    Let $\mathcal{S}$ be the set of the SCM such that $P_{obs}(S) = p$. Let $a= \min_{S\in\mathcal{S}} f(p,S), b= \max_{S\in\mathcal{S}} f(p,S)$.
    Then $f|_p:\mathcal{S}\rightarrow [a,b]$ is surjective. 
\end{restatable}

\subsection{Decomposing  %
Causal Effect Uncertainty }
\label{epi-uncertain}

Given the uncertainties in estimating the ATE, we now define the challenges in decision-making.

\begin{definition}[Causal Decision]
    A causal decision with $ATE(x)$, given confidence set $\mathcal{C}(D)$, is unambiguous for action $x$ if all SCM $S$ entailing some $p \in \mathcal{C}(D)$ agree in sign, i.e., either $ATE(x) \geq 0$ for every such $S$, or $ATE(x) \leq 0$ for every such $S$. The decision is ambiguous if there exist two SCM $S_1, S_2$ entailing $p_1, p_2 \in \mathcal{C}(D)$ such that $ATE(x) > 0$ in $S_1$ and $ATE(x) < 0$ in $S_2$.
\end{definition}

When a decision is ambiguous, a natural next step is to collect more samples, but for a non-identifiable query, decision might still be ambiguous even with infinite data due to the non-ID uncertainty. In the next section, we decompose the uncertainty for decision-making.

The main question we pose in this paper is: \emph{Can we know when collecting more samples will unlikely lead to informed decision-making?} To answer this question, we need to gain insight into two different types of epistemic uncertainty without knowing the observational distribution exactly: one modeling uncertainty due to unobserved confounding, and one modeling the uncertainty due to finite samples.

Our proposed method distinguishes the two uncertainties and provide a lower bound for the uncertainty that is unlikely to be reduced by collecting more data.
This helps us decide the right conclusion about
\begin{wrapfigure}[13]{rh!}{0.4\textwidth}
\vspace{-1.5em}
\includegraphics[width=1\linewidth]{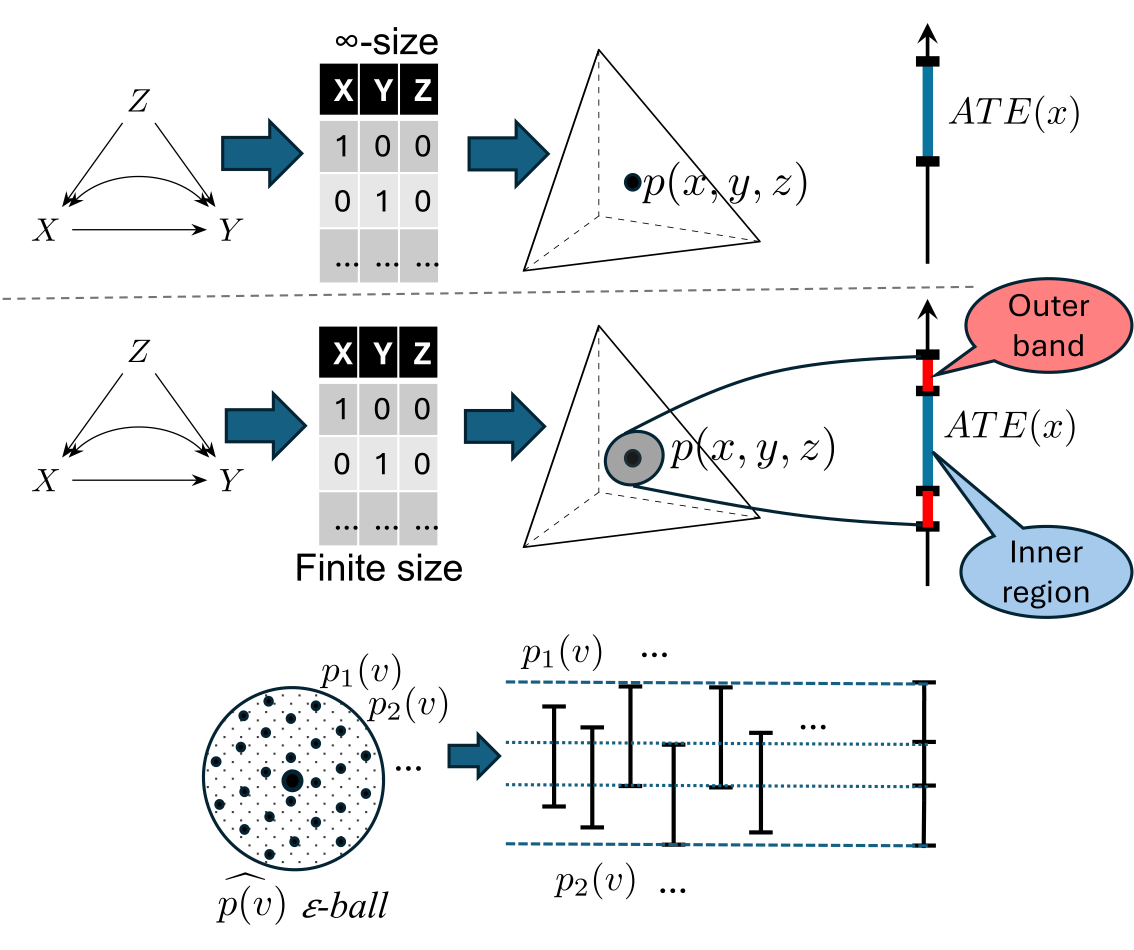}
\vspace{-2em}
\caption{Inner and outer causal effect bounds when the effect is not ID.}
\label{fig:bound-illustration}
\end{wrapfigure}
data collection practices for decision-making. 
 
\textbf{Assumptions}:
We assume \textit{i)} {discrete variables }, \textit{ii)} Semi-Markovian SCM and \textit{iii)} access to the acyclic directed mixed graph (ADMG). See Section~\ref{thr-anl} for details.
In this section, we discuss a general scenario in which we compare the effects of multiple (two or more) actions to choose the best action. 
 We discuss the case of optimizing only a single action in Section~\ref{sec:sing-act}.  We formally introduce the two types of uncertainties below.

\label{sec:decide}

\begin{definition}[\epSample and \epNonID Uncertainty (Figure~\ref{fig:bound-illustration})]
\label{def:uncertainty-mult}

    Let $\hat{P}$ be an empirical estimation of the joint distribution and a confidence set $\CR$ be such that $\CR$ covers the true $P$ with probability of at least $1-\alpha$. 
    Define the set $\mathcal{S}_P=\{S\text{ an SCM} \mid P_\text{obs}(S)=P,  {S \models G} \}$ and 
    $  \mathcal{I}_P(x) = \{ATE(x) \mid S\in\mathcal{S}_P\}$. 
    Then define the intersection $\bigcap_{{P}\in\CR} \mathcal{I}_{{P}}$ to be the \textbf{\epNonID uncertainty}, and the set difference $\bigcup_{{{P}}\in\CR} \mathcal{I}_{{P}}\setminus \bigcap_{{{P}}\in\CR} \mathcal{I}_{{P}}$ to be the \textbf{\epSample uncertainty}.
\end{definition}

Intuitively, for each distribution $P$ in the confidence set $\mathcal{C}$, we obtain a set of ATE values $\mathcal{I}_P(x)$ from a set of SCMs  $\mathcal{S}_P$ that matches $P$ and consistent with causal graph $G$.
We construct \epNonID uncertainty  by taking intersection of $\mathcal{I}_P$ for all $P\in \mathcal{C}$ and claim the following proposition: 

\begin{restatable}{proposition}{irred}
    \label{prop:irred-samp-ate}
    Let $P^\star$ be the true observational distribution, $\mathcal{C}(D)$ a confidence set satisfying $\Pr\left(P^\star \in \mathcal{C}(D)\right) \geq 1-\alpha$, and $\mathcal{I}_{P}(x)$ the causal-effect identified set for observational distribution $P$. On the coverage event $P^\star \in \mathcal{C}(D)$, 
    (i)~$\bigcap_{P \in \mathcal{C}(D)} \mathcal{I}_{P}(x) \subseteq \mathcal{I}_{P^\star}(x)$; and (ii)~if the inner region crosses the decision boundary, i.e., $L_{\mathrm{inner}} < 0 < U_{\mathrm{inner}}$, then, with probability at least $1-\alpha$, the population identified set $\mathcal{I}_{P^\star}(x)$ also crosses zero.
\end{restatable}

Note that (i) should be interpreted as a high-probability inner certificate of population non-identification obtained from the current dataset $D$. By (ii), if the inner region $\bigcap_{P\in\mathcal{C}(D)} \mathcal{I}_{P}(x)$ crosses zero, then with probability at least $1-\alpha$, the population identified set $\mathcal{I}_{P^\star}(x)$ crosses zero, and the decision would remain ambiguous even if infinitely many additional observational samples revealed $P^\star$ exactly. Thus, we should \textit{\textbf{obtain}} additional variables to reduce this uncertainty.
This result does not require empirical distributions or confidence sets obtained from future datasets to be identical or nested.
Next, we construct \epSample uncertainty for ATE values: if $|\bigcup_{{{P}}\in\CR} \mathcal{I}_{{P}}\setminus \bigcap_{{{P}}\in\CR} \mathcal{I}_{{P}}|\neq 0$, it is possible that collecting more data points can disambiguate the causal effect. Thus, we \textit{\textbf{collect}} more samples to reduce this uncertainty. Formally:
\begin{proposition}
\label{prop:red-samp-ate}
    There exists SCMs $S\in \mathcal{S}_P$ where the uncertainty in the causal decision-making problem due to  {\epSample uncertainty} in   $ATE(x)$ can be reduced by increasing the dataset size.
\end{proposition}

To represent the uncertainty regions in Proposition~\ref{prop:irred-samp-ate}- \ref{prop:red-samp-ate} numerically, we define four quantities:
\begin{definition}
\label{appex-def-ate}
For some action $x\in\mathcal{X}$,
we define these quantities to compute uncertainty regions. 
\begin{equation}
\label{eq:4quant-mult}
\begin{alignedat}{3}
\Uatex
&:= \sup_{{P} \in \CR(\hat{{P}})} \max_{S \in \mathcal{S}_P} ATE(x)
&\qquad\qquad&
\Latex
&:= \inf_{{P} \in \CR(\hat{{P}})} \min_{S \in \mathcal{S}_P} ATE(x),
\\
\Upatex
&:= \inf_{{P} \in \CR(\hat{{P}})} \max_{S \in \mathcal{S}_P} ATE(x)
&\qquad\qquad&
\Lpatex
&:= \sup_{{P} \in \CR(\hat{{P}})} \min_{S \in \mathcal{S}_P} ATE(x).
\end{alignedat}
\end{equation}
\end{definition}
\begin{restatable}{theorem}{intersectmult}
\label{thm:intersect-mult}
    Given a confidence set $\CR$ and ATE values $\mathcal{I}_P$ 
    , $\bigcap_{P\in\CR}\mathcal{I}_P = [\Lpatex,\Upatex]$ if not empty.
\end{restatable}

\begin{restatable}{corollary}{outerregionmult}
\label{outer-region-mult}
     Given a confidence set, $\CR$ and a set of average treatment effect values,   $\mathcal{I}_P$ (Def.~\ref{def:uncertainty-mult}), %
     $$\bigcup_{P\in\CR} \mathcal{I}_P\setminus \bigcap_{P\in\CR} \mathcal{I}_P=[\Latex,\Lpatex]\cup[\Upatex,\Uatex].$$
\end{restatable}
Thus, we can access the sample and non-id uncertainty by estimating four quantities for ATE as \Cref{eq:4quant-mult}.
Now that we have connected uncertainty regions with bounds obtained by mini-max optimization in Equation~\ref{eq:4quant-mult}, we can select the best action or suggest how to approach the decision with Theorem~\ref{thm:decision-ate}.

\begin{restatable}{theorem}{decisionate}
\label{thm:decision-ate}
    For a causal decision-making problem with $ATE(x)$ as the decision metric, let $(\Uatex, \Upatex, \Lpatex, \Latex)$ be estimated from the data. The decision of $X=x$ as the best action (or cannot be the best action) is unambiguous if $\Latex \geq 0$ (or $ \Uatex < 0$). The decision is ambiguous and, with probability at least $1-\alpha$, cannot be improved with more data if $\Lpatex<0<\Upatex$.
\end{restatable}

\textbf{From ambiguous to unambiguous decision:}
Using Theorem~\ref{thm:decision-ate}, we can reach the final decision through a sequence of three moves.
\textit{1. Return:} if $0<\Latex$ then we return $do(X=x)$ as the best action.
\textit{2. Observe}: if $
\Lpatex
< 0 < 
\Upatex
$, then we observe additional variables for example instrument variables.
Observing additional variables gives us a small causal effect set $\mathcal{I}'_P$, for each $P\in \CR$. Thus, for fixed $\CR$, we get a narrower inner region, $[
\Lpatex,
\Upatex
] =  \bigcap_{P\in\CR}\mathcal{I}_P'$.
Also, the average non-id bound width will reduce since  $1/|\CR|  \sum_{P\in \CR}|\mathcal{I}'_P| \leq 1/|\CR|  \sum_{P\in \CR}|\mathcal{I}_P|$.
\textit{3. Collect}: In all other scenarios, we are unsure about the source of uncertainty (low sample size or non-identifiability) and we collect more samples, with the assumption that obtaining samples with additional variables is more challenging compared to collecting samples of the same set of variables. More samples typically provide us a smaller confidence set $\CR'(\hat{P})$ that concentrates around $P^\star$.
By Proposition~\ref{prop:irred-samp-ate} and~\ref{prop:irred-samp}, the inner region, $[\Lp_{x}, \Up_{x}] = \bigcap_{P\in\CR'}\mathcal{I}_P$ (Theorem~\ref{thm:intersect} and~\ref{thm:intersect-mult}), remains contained in $\mathcal{I}_{P^\star}$ with probability at least $1-\alpha$,  typically widens toward $\mathcal{I}_{P^\star}$, while the outer band, $[\Lo_x,\Lp_x]\cup[\Up_x,\U_x]=\bigcup_{P\in\CR'} \mathcal{I}_P\setminus \bigcap_{P\in\CR'} \mathcal{I}_P$ (Corollary~\ref{outer-region} and~\ref{outer-region-mult}), typically narrows.
The \textit{collect} and \textit{observe} moves continue until we reach the \textit{return} move and decide the best action.
The decomposition of uncertainty methods we discussed is also applicable to other causal queries such as $\mathbb{E}[Y|do(x)]$ which we discuss in \Cref{sec:sing-act}.

\begin{figure}[t!]
\vspace{-5mm}
\centering
\begin{minipage}[t]{1\textwidth}
\footnotesize
\begin{algorithm}[H]
\footnotesize
\caption{$\mathrm{Explore }$ $\epsilon$-$\mathrm{ball}$}
\label{alg:decide}
\begin{algorithmic}[1]
\footnotesize

\State \textbf{Input:}
data $\{\mathbf v_k\}_{k=1}^n$, graph $\mathcal{G}$, small interval $\epsilon_s$
\State \textbf{Output:} Decision/Collect/Observe

\State Initialize $\gamma = 0$ for multiple actions (ATE). \Comment{Decision boundary value}

\While{Not Decided}
    \State Candidates $= \mathrm{Construct}$ $\epsilon$-net(data, $\mathbf{V}, \epsilon_s$) \Comment{Candidate joint distributions consistent with input data}

    \For{$P_{\epsilon}(\mathbf{V}) \in \text{Candidates}$} \Comment{Explore the confidence set}
        \For{$x \in \operatorname{sup}(X)$} \Comment{Optimize each action in the support}
            \State $U_x,L_x \gets \mathrm{RelaxedDCM}(
            P_{\epsilon}(\mathbf{V})
            , \mathbf V,\mathcal G, X=x, \epsilon_s)$ \Comment{Execute Alg~\ref{alg:dcm} to  compute bounds}

            \State $\U_x, \Up_x \gets \max(\U_x, U_x), \min(\Up_x, U_x)$ \Comment{Update upper bounds (maxmax,minmax)}
            \State $\Lo_x, \Lp_x \gets \min(\Lo_x, L_x), \max(\Lp_x, L_x)$ \Comment{Update lower bounds (minmin, maxmin)}
            \label{alg:maxmax}
            \label{alg:minmin}
        \EndFor
    \EndFor

    \If{$\gamma \notin [\Lo_x, \U_x]$} \Comment{Decision is conclusive}
        \State \textbf{Return} $X=x$ as conclusive; $\operatorname{do}(x)$ is better/worse compared to other actions.
        \label{alg:best_act}

    \ElsIf{$\gamma \in [\Lpatex, \Upatex]$} \Comment{Additional samples are futile. Need more variables}
        \State \textbf{Observe} variables
        \label{alg:conc_instru}

    \ElsIf{$\gamma \in [\Lo_x,\Lp_x]\cup[\Up_x,\U_x]$} \Comment{Collect samples to reduce sample uncertainty}
        \State \textbf{Collect} more samples
        \label{alg:collect-smp}
    \EndIf
\EndWhile

\end{algorithmic}
\end{algorithm}
\end{minipage}
\vspace{-3mm}
\end{figure}

\subsection{DCM for Estimating Sample and Non-ID Uncertainty}
\label{sec:dcm-non-id-unc}

We introduce three algorithms: Alg.~\ref{alg:decide} calls  Alg.~\ref{alg:enet} to construct the $\epsilon$-net and then executes Alg.~\ref{alg:dcm}: RelaxedDCM
for each distribution in the $\epsilon$-net to find the non-id causal effect bound. Finally, Alg.~\ref{alg:decide}: $\mathrm{Explore }$ $\epsilon$-$\mathrm{ball}$ combines these bounds to construct  $\four$ to decide the next move.

\textbf{Algorithm~\ref{alg:dcm}: Optimizing Relaxed-DCM}:
Given a fixed distribution $P_{\epsilon}$, in Alg.~\ref{alg:dcm}, we train deep causal models (see Def.~\ref{def:scm})
to learn $P_{\epsilon}$ while optimizing the causal effects. 
Unlike existing works that attempt to match $P_{\epsilon}$ exactly, we use a Lagrangian duality-based optimization~\citep{bui2022unified,gao2023distributionally} to keep the observational distribution entailed by the DCM $P_{\theta}$, within a small region around input distribution, $P_{\epsilon}$ (Alg~\ref{alg:dcm}: line 7). 
The dual update parameter successfully finds the right amount of regularization and maintains the desired distributional distance (Alg~\ref{alg:dcm}: line 9). We add the causal effect magnitude with the distributional loss terms and optimize the parameters to maximize or minimize the effects(Alg~\ref{alg:dcm}: lines 8, 10).

\textbf{Algorithm~\ref{alg:enet}: 
Construct $\epsilon$-net} :
Since the given dataset might be sampled from any distribution in the confidence set $\CR$ of the empirical distribution $\hat{P}$, we need to solve the max-max, max-min, min-max and min-min problems (Equation~\ref{eq:4quant}, \ref{eq:4quant-mult}) by searching over $\CR$.  
To obtain the confidence set in practice,  we factorize the empirical distribution $\hat{P}(\mbf{v})$ into conditional distributions, i.e., $\hat{P}(\mbf{v}) = \Pi_{v\in \mbf{v}} \hat{P}(v_i|v^{\pi_i-1})$ and 
construct non-asymptotic confidence interval for variable with $m$ total number of configurations (exponential in alphabet size), from $n$ samples w/ error probability $\alpha=0.05m$,  using Hoeffding inequality~\citep{Hoeffding:1963} and Bonferroni correction, i.e.,  $\hat{P}_n(v_i=1|v^{\pi_i-1}) \pm  \sqrt{\frac{\ln(2/\alpha)}{2n}}$. See Appendix~\ref{derivation} for derivation.
If we pick one point from each interval to form a joint distribution and optimize the causal effect, we obtain the maximum and minimum effect for that specific distribution. To obtain the four quantities in Equation~\ref{eq:4quant}, \ref{eq:4quant-mult}, we have to search over all constructed intervals and optimize the causal effect while matching each formed distribution.

Note that solving a max-min and min-max problem by searching over these confidence intervals is not convex in general. %
To address this scenario, we resort to a heuristic approach motivated from $\epsilon$-net~\citep{gonzalez1985clustering}. To implement $\epsilon$-net, we cover the metric space of observational distributions $M$ consistent with the input dataset, with open balls having centers in $X\subset M$ such that every distribution in $M$ is within $\epsilon$ distance of at least one distribution in $X$.
For our purpose, we cover each confidence interval with smaller intervals of width $\epsilon_s$. If $\epsilon$ is the interval width, this would give us $\ceil{\epsilon/\epsilon_s}$ number of centroids.
We {uniformly} pick centroid for each conditional and form a joint distribution. 
After optimizing the DCM for these distributions, we take the minimum (or maximum) of all optimized values to approximate the minmax (or maxmin). 

\textbf{Algorithm~\ref{alg:decide}: {Explore } $\epsilon$-{ball} }:
After collecting $[L_x,U_x]$ bounds from each distribution in the confidence set, we follow Equation~\ref{eq:4quant} or~\ref{eq:4quant-mult} to obtain the inner $\inner$ region and the outer $\outr$ band. If the bounds are completely separated for ${do(}x_0)$ and ${do(}x_1)$, we return the best action(Alg~\ref{alg:decide},line:~\ref{alg:best_act}). If the inner region intersects, we conclude with high probability that additional variables must be observed (Alg~\ref{alg:decide},line:~\ref{alg:conc_instru}).
In this paper, we suggest observing instrument variables $I$ since an instrument can reduce the non-id region~\citep{balke1996universal}. Whether other variables can reduce the causal effect bounds is an open area to explore. In all other cases, we conclude collecting more samples (Alg~\ref{alg:decide},line:~\ref{alg:collect-smp}).
The conclusion holds with probability at least $1-\alpha$, provided that the $\epsilon$-net adequately covers the confidence set and the neural networks are optimized well; otherwise, the resulting bounds may be inaccurate. See Appendix~\Cref{sec:alg_detail} for details.
\section{Experimental Analysis}
In this section, we evaluate our algorithm in three simulated setups and on one real-world Parent labor supply dataset. 
We demonstrate how our algorithm performs with access to more samples and  observed variables and how it adapts to large graphs. Finally, we analyze the behavior of our algorithm for mis-specified causal graphs (Section~\ref{sec:mis-specify}). In all setups, we obtain $\four$ by choosing $\sim100$ distributions from the $\epsilon$-net and training the neural nets on 5000 samples from each. Code: \url{https://github.com/Musfiqshohan/UA-DCM}.
\subsection{Performance on Simulated Datasets}
\label{sec:scm1}
\label{sec:scm2}
\begin{figure}[t!]
\hspace{-5mm}
\vspace{-6mm}
\centering
\begin{minipage}{0.50\textwidth}
\centering
\includegraphics[width=1\linewidth]{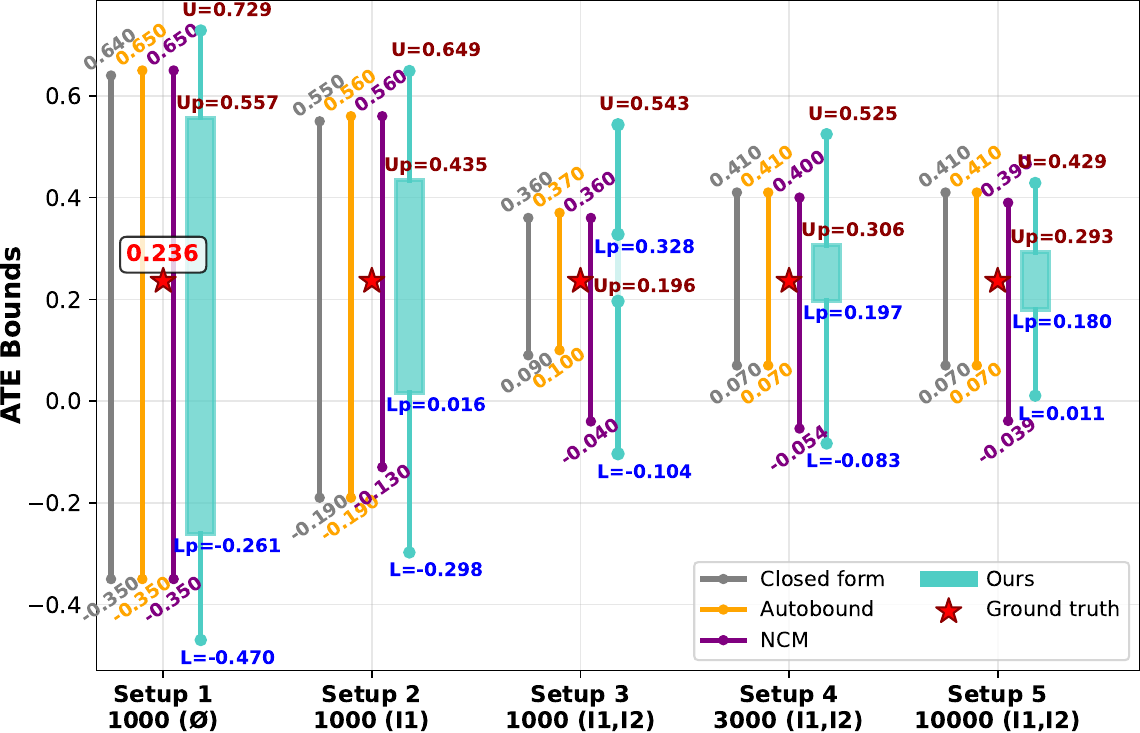}
\captionof{figure}{Bounds comparison for the multiple IV.}
\label{tab:iv_multiple}
\end{minipage}
\hfill
\begin{minipage}{0.48\textwidth}
\centering
\includegraphics[width=1\linewidth]{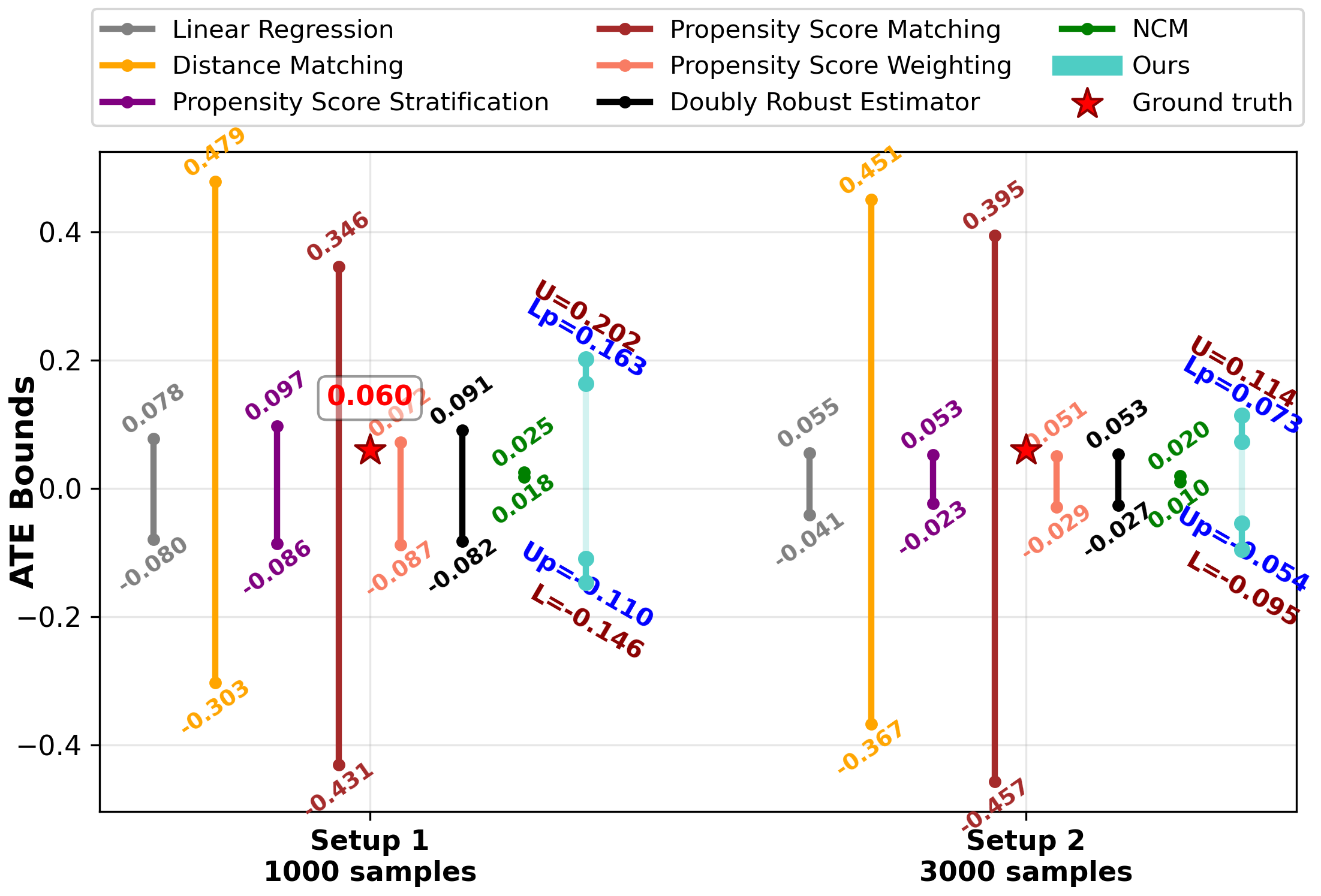}
\captionof{figure}{Bounds comparison in the ID setting.}
\label{tab:iv_two_cnf}
\end{minipage}
\end{figure}
\textbf{SCM 1 {\myalgo} correctness \& utility):}
In Figure~\ref{tab:iv_multiple}, we show the next moves and consequences of our algorithm for different sample sizes.
In setup 1, for $N=1000$ samples of $D[X,Y]$, we observed that all three baselines, i.e., closedform, autobound and NCM output similar ATE bounds: $[-0.35,0.64]$. They estimate $P(X,Y)$ from data and use it as the true distribution to calculate the bounds.  These baselines only indicate that the true ATE lies within a large bound but provide no additional information on how to reduce their bounds to make the best action decision. On the other hand, in \myalgo, since $\gamma\in [\Lpate, \Upate]$, i.e., $0\in [-0.26, 0.56]$, we conclude that we must $\textit{observe}$ additional variables. 
With probability at least $1-\alpha$, additional samples will not resolve this ambiguity.

We illustrate the consequence of observing additional variables in setup 2 where we consider an additional instrument variable. We observe that for 1000 samples of $D[I_1, X,Y]$, all baselines bounds reduce to $[-0.19, 0.56]$. For \myalgo, both the inner region and the outer band shrink, especially the inner bound from $[-0.26, 0.56]\rightarrow [0.02, 0.44]$, indicating the reduction in the non-id region. However, as the solid region $[\Lpate, \Upate]$ does not contain $0$, we do not know if the current uncertainty is due to unobserved confounders or low sample size. Thus, based on the feasibility of the setup, we have to select the next move. For our experiment, we conclude,  \textit{observing} another instrument  $I_2$.

In setup 3, we show bounds obtained from all baselines and our algorithm \myalgo for 1000 samples of the dataset $D[I_1, I_2, X,Y]$, and in setup 4, we show results when we collect 3000 samples. In setup 3, the bounds reduce significantly while changing slightly in setup 4. \myalgo says that the outer band is $[-0.08, 0.20]$, which is very close to zero. Thus, collecting more samples will reveal that $do(X=1)$ is a better action compared to $do(X=0)$. Finally, we observe in setup 5 that with 10k samples and $90\%$ confidence, the minmin value is above $0$. Now, we can decide safely that $do(X=1)$ is the best action for all distributions consistent with the finite input samples.
We perform additional analysis in Appendix~\ref{sec:alg_detail} with various confidence levels and various sample sizes.

\textbf{SCM 2:}
We consider only observed confounders %
(Figure~\ref{fig:two-obs-cnf}) with no structural uncertainty. 
We compare with 6 baselines that are designed to provide point-wise estimations.
To be fair, we perform bootstrapping and obtain bounds from them.
 We also compare with NCM which outputs ATE bound.
 
\textbf{Results:}
We observe that for 1000 samples, distance matching and propensity score matching output a very wide bound, ex: $[-0.457,0.395]$. Other baselines provide tighter bounds and
NCM outputs $[0.018, 0.025]$.
All of these methods estimate a probability table from the input $1000$ samples and estimate the causal effect based on that.
Even with bootstrapping, some of the bounds do not contain the true ATE$=0.06$.
As most of the baselines contain $0$ within it, they suggest that no action can be decided as the better one. No other insight about uncertainty reduction can be extracted from them.
On the other hand, \myalgo provides an outer bound $[-0.14, 0.20]$ for 1000 samples and an outer bound $[-0.09, 0.11]$ for 3000 samples. This indicates that the ATE obtained from the true distribution will be contained in this bound. Also, note that there is no gap between $\Upate$ and $\Lpate$ (Up,Lp in the figure), which implies that there is no non-id uncertainty here, hinting  us to collect more samples.

\textbf{Scaling Experiment: }
 To scale our algorithm for large graphs, we utilize two practical solutions (details in Appendix~\ref{fig:scale}). 
 1. We factorize the joint distribution according to the causal graph and match the factorized conditional distributions instead of the full joint.
 2. We ignore low-probability joint combinations and redistribute their probability mass.
 This reduces the slicing of the alpha value to a large number of combinations that do not meaningfully change the causal effect.

\textbf{Observation:}
In Figure~\ref{fig:large-graph}, we show the algorithm's performance on graphs with $8$, $12$, and $18$ variables, each using $10k$ samples.
For the $8$-variable setup, pointwise identification of the causal effect collapses the inner region; thus, there is no non-identifiable region. 
For each scenario of $8$-vars, $12$-vars and $18$-vars, our bounds enclosed the true ATE values and the outer bound was completely above zero, implying that $do(X=1)$ is the best action for these problem instances. 
 Thus, our algorithm  scales to a relatively larg number of variables and offers bounds that enclose the true ATE.

\subsection{Real Data Experiments: Parents' Labor Supply}

\label{sec:appex-lbr}

\textbf{Setup:} 
We apply our algorithm on the  Parents' Labor Supply dataset\citep{angrist1996children} which contains contains 92k family records of their demographic and work details. 
Here, the goal is to evaluate the  causal effect of having more than two children (treatment $X$:yes/no) on
mother's labor supply (outcome $Y$: working or not) with respect to the population, i.e., $P(y=1|\Do(x=1))$ vs $P(y=1)$.
We follow~\citep{angrist1996children}
and consider whether first two children are of same sex or not as an instrument variable $I$ since this significantly motivates people to have the third child. Please see~\ref{appex:parent-labr} for further justification.
 We illustrate our algorithm performance on 1000, 2000 and 92k (full dataset) with and without an instrument.

\begin{figure}[t!]
\centering
\vspace{-5mm}
\hspace{-3mm}
\begin{subfigure}{0.28\textwidth}
\includegraphics[width=1\linewidth]{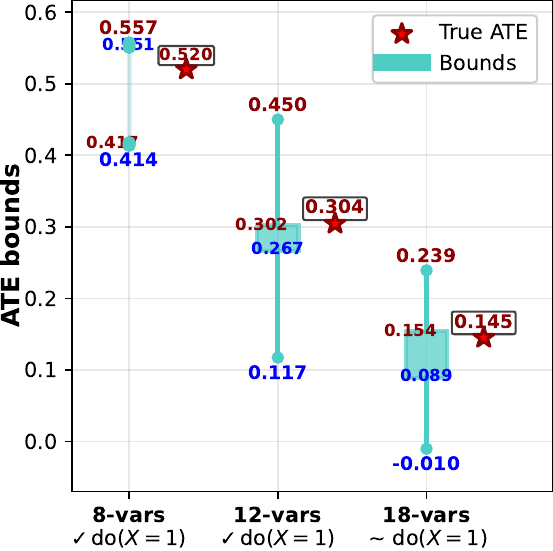}
\caption{Performance on large graphs}
\label{fig:large-graph}
\end{subfigure}
\begin{subfigure}{0.46\textwidth}
\centering
\includegraphics[width=1\linewidth]{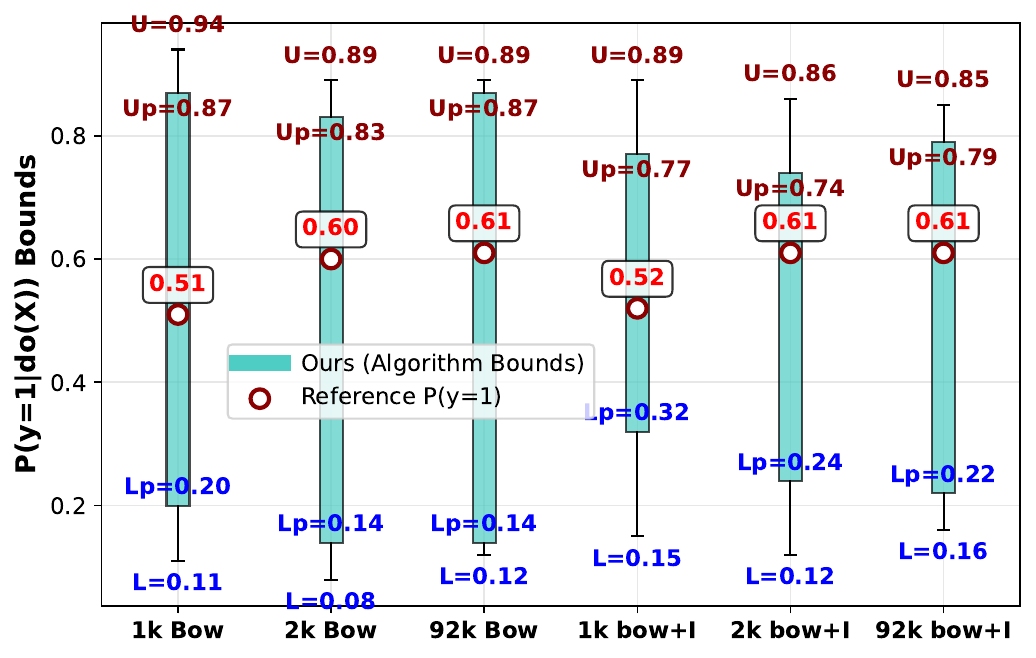}
\caption{$do(x)$ bounds on Parents Labor Supply (PLS).}
\label{tab:lbr_sup}
\end{subfigure}
\begin{subfigure}{0.25\textwidth}
\includegraphics[width=1\linewidth]{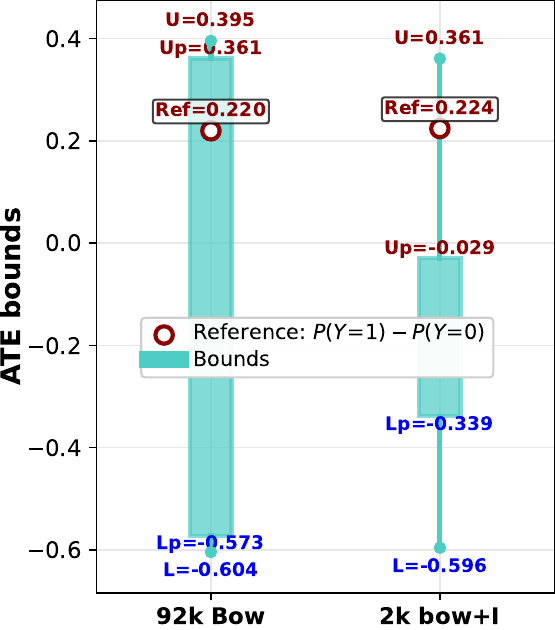}
\caption{ATE bounds for PLS .}
\label{fig:pls-ate}
\end{subfigure}
\caption{\myalgo performance on synthetic and real-world datasets.}
\label{fig:scal}
\vspace{-3mm}
\end{figure}

\textbf{Results}: In Fig~\ref{tab:lbr_sup}, as we increase sample size from 1k to 92k, the inner region expands as $[0.20, 0.87] \rightarrow [0.14, 0.87]$, and the
total width of the outer band shrinks as $0.16 \rightarrow 0.04$. However,
$P(y = 1)$ is still contained in the non-id inner region. Thus, with high probability, collecting more samples will
not help, and we include the sex of the first two children
as an instrument. For $1000$ samples, as we move
from no instrument to one instrument (1k bow $\rightarrow$ 1k bow+I), the non-id region is
reduced as $[0.20, 0.87] \rightarrow [0.32, 0.77]$. Although the bound
does not exclude $P(y)$, it shows we need to keep
observing variables to reach a conclusion.
 We compare the \textit{ATE} bounds for the $92k$ bow and $2k$ bow + I setups in Figure~\ref{fig:pls-ate}. The $92k$ samples eliminate almost all sample uncertainty, leaving only the non-identification region. In contrast, in the $2k$ bow + I setup, the non-id region $[-0.57, 0.36]$ decreases significantly to $[-0.33, 0.22]$ after including the instrumental variable. The outer band is extended due to limited samples. Our results are consistent with~\citet{angrist1996children}, who applied Two-Stage Least Squares and found a negative impact: $-0.12$ on "worked for pay" and $-5.68$ on "weeks worked per year". Our ATE bounds also suggest such a negative effect.
\section{Conclusion}
In this paper, we proposed a novel decomposition of epistemic uncertainty in causal effect estimation from finite data. This decomposition can be used to conclude with high probability collecting more samples will not help with deciding the best action. We further develop a neural causal modeling approach to approximate this decomposition in practice. 
A limitation of the current method is computational scalability: the decomposition requires optimizing causal effect bounds over a confidence set of observational distributions, which we approximate with an $\epsilon$-net. The search over candidate distributions can grow exponentially with the number of variables in the worst case. 
Developing more efficient optimization procedures is an important direction for future work.

\subsection*{Acknowledgments}
This research has been supported in part by NSF CAREER 2239375, IIS 2348717, Amazon Research Award, Adobe Research and Intuit AI Research.

\bibliographystyle{preprint}
\bibliography{ref}

\clearpage

\appendix

\section{Limitations and Future Works}
\label{sec:limitations}

We highlight the main limitations of our framework to help practitioners judge when its recommendations can be trusted.

\paragraph{Correct causal graph.} Our framework assumes a correctly specified ADMG (Assumption~3). If the graph omits edges that affect identifiability, the framework may misclassify a non-identifiable query as identifiable and consequently produce incorrect predictions.

\paragraph{Conservative confidence set.} The Hoeffding and Bonferroni based confidence set (Sec.~\ref{sec:dcm-non-id-unc}) trades tightness for coverage. As a result, it can underestimate the non-identifiable region and delay the ``Observe'' recommendation when the query is non-identifiable. To mitigate this, we provide a formal analysis of how many additional samples suffice for the predicted non-identifiable region to recommend ``Observe'', which gives practitioners guidance on the sample complexity required to recover from such underestimation.
Moreover, distributions in the confidence set that are far from the empirical distribution may violate the constraints induced by the causal graph (e.g., the instrumental inequalities); for the IV graph, we discard such distributions before exploration, while enforcing these constraints for general graphs with latent variables remains an open problem (Appendix~\ref{sec:graph-constraints}).

\paragraph{Optimization of the bounds.} Performing the effect bound decomposition requires solving challenging min-max and max-min problems, which we approach using an $\epsilon$-net. The size of this net grows with the desired resolution, adding computational cost. Moreover, neural network optimization may be imperfect in practice and yield bounds tighter than the true ones.

\paragraph{Future work.} We aim to resolve the above limitations i.e.  develop more efficient min-max optimization procedures, reduce the complexity of the $\epsilon$-net, and relax the reliance on a fully specified causal graph.

\section{Broader Impact}
\label{sec:brod-imp}
In this paper, we propose an algorithm that can determine
when {with high probability} collecting more samples will not help determine the
best action. This can guide practitioners to realize when
increasing the number of participants in the current setup,
i.e., collecting additional data samples is not being helpful
anymore. Our algorithm will advice to collect more variables or lean towards a randomized study for best action
identification. As a result, our algorithm can be used as a
too for strategic planning of randomized control study and
reduce cost for experimentation. Additionally, our work
can advance the field of AI by improving machine learning methods and make them fair by providing the causal
analysis of different actions.

\section{Related Work}

\textbf{Partial Identification of Causal Effect.} When the causal effect is not identifiable from observational data, one can identify a set of causal effects from the SCMs that are compatible with the observational data. For the discrete variables, \citet{balke1994counterfactual, balke2022probabilistic} proposed a method using response variables to partition the latent variable into finite states and estimate the bounds of causal queries. \citet{tian2000probabilities} used the response variable to get bounds of counterfactual queries from observational and interventional data. Many works focus on getting narrower bounds by utilizing additional information from the graph. \citet{balke1997bounds} obtained tighter bounds with instrumental variables. \citet{duarte2024automated} proposed an automated method to derive bounds for causal queries in arbitrary graphs. 

There are works that are closely related to ours utilizing a posterior sampling approach for partial identification. \cite{zhang2021non} introduces canonical nonparametric SCMs for bounding causal effects in arbitrary causal diagrams with categorical observed variables. \cite{zhang2022partial} extend this approach to counterfactual queries using observational and experimental data, and provide MCMC-based credible intervals for approximating partial identification bounds. These works are complementary to ours: they provide finite sample guarantees for the MCMC approximation of credible intervals conditioned on the observed data, whereas our framework studies the decision problem induced by the finite sample uncertainty. This line of work also suggests an alternative but related decision-making problem: given a limited sample budget, determine whether the available data are sufficient to reach an unambiguous decision or whether additional samples within the budget are unlikely to resolve the ambiguity. Developing such a criterion requires determining the number of samples needed for the estimation to reach a given bound width and propagating finite-sample uncertainty through partial identification bounds. This might be an interesting future direction connecting the posterior sampling method with the decision problem.

Another line of work tries to obtain narrower bounds of causal effect utilizing information from the unobserved confounders. \citet{li2022bounds} proposed using nonlinear programming to get bounds when the latent confounder is partially observed. \citet{jiang2023approximate} proposed a convex programming formulation to get tighter bounds of the causal effect when the entropy of the latent confounder is known. \citet{li2023epsilon} derived bounds in closed form when the marginal distribution of the latent confounders is known. \citet{zhang2024tight} proposed a method that finds the tight identifiable region of causal effect given the marginal distribution of latent confounders. 

\textbf{Causal Inference with Neural Models}. 
Many researchers have explored the use of neural networks in causal inference. \citet{kocaoglu2018causalgan} introduced CausalGAN that produces interventional images with a GAN trained with image data. \citet{yang2021causalvae} proposed CausalVAE to learn the causally related representations from the data. \citet{xia2021causal, xia2022neural} proposed neural causal models that find the causal and counterfactual queries by maximizing and minimizing the query under the semi-Markovian setting. \citet{chao2023modeling} proposed a diffusion-based approach to model the causal and counterfactual queries. \citet{rahman2024modular} proposed a modularized training algorithm to train a causal generative model with latent confounders for high-dimensional variables. \citet{rahman2024conditional} proposed a recursive algorithm that uses a set of conditional generative models to sample from any identifiable interventional distributions. 

\textbf{Uncertainty Quantification}. 
Uncertainty quantification is a crucial problem in deep learning~\citep{gal2016uncertainty}, especially in applications where reliable prediction is essential. \citet{hora1996aleatory} introduced the notion of aleatoric and epistemic to distinguish uncertainties for the risk models in applications such as decision making. \citet{abellan2006disaggregated} defined the total uncertainty based on the Shannon entropy of a credal set. \citet{sensoy2018evidential} quantified the epistemic uncertainty of the classification model with the Dirichlet prior. In general, the uncertainty is assumed to have an additive representation between the total uncertainty (TU), epistemic uncertainty (EU), and aleatoric uncertainty (AU), and it is generally hard to disentangle the EU and AU~\citep{hullermeier2021aleatoric, hullermeier2022quantification}. \citet{melnychuk2024quantifying} proposed a method for quantifying aleatoric uncertainty in individualized treatment effects by deriving sharp bounds on the conditional distributions of the treatment effect. \citet{marmarelis2024ensembled} introduced an approach for constructing causal effect intervals with hidden confounding. \citet{jiang2024conditional} proposed using entropy as a sensitivity parameter to obtain bounds in the IV graph with assumption violations.

\textbf{Decision Making}. Decision-making is highly related to the bandit problem, which aims to find the action that optimizes the rewards~\citep{cho2024reward,huang2023bandit,ji2025multi}. In the context of causal inference, although the causal effect is closely related to the decision-making problem, \citet{fernandez2022causal} suggested that they are not exactly equivalent, where the main distinction is their estimand of interest. \citet{wei2023approximate} proposed an algorithm that utilizes the causal graph to balance exploration and exploitation in sequential decision-making problems with latent confounders. \citet{kallus2018confounding} studied the decision-making problem with observational data with the presence of latent confounders. They showed that with a well-specified uncertainty set, the proposed robust policy learning is no worse than the baseline. \citet{jesson2020identifying} introduced a method to assess the uncertainty in the neural network model, and demonstrated that the proposed method is able to handle the positivity violation. Sensitivity analysis is also closely related to the uncertainty quantification in causal inference. \citet{frauen2023neural} proposed a neural framework for sensitivity analysis that is compatible with a large class of existing sensitivity models. \citet{jesson2020identifying} studied the decision-making problem under the non-overlapping situation. \citet{frauen2025treatment} propose a optimal decision-making approach based on two-stage CATE estimators.

\textbf{Causal Bandits}. A line of work uses offline causal knowledge 
 to derive bounds for actions first and then proceed to the online phase.{\citet{zhang2017transfer}} derive two-sided bounds on each arm's expected reward from a source agent's causal model and batch data, and feed these bounds into a modified UCB algorithm (B-kl-UCB) that filters and truncates arms accordingly. {\citet{bellot2023transportability}} instead use a known selection diagram, i.e., a causal graph annotated with where source and target environments differ, to build an informative Bayesian prior over reward distributions from batch source data, which then seeds posterior sampling exploration in the target environment. A related line studies where in the causal graph to intervene: {\citet{lee2018structural}} introduce structural causal bandits, which use graphical criteria derived from do-calculus (possibly-optimal minimal intervention sets, or POMIS) to prune the arm set before play begins; {\citet{lee2019structural}} extend this framework to non-manipulable variables, {\citet{park2025structural}} generalize it from a known DAG to a Markov equivalence class (PAG/MAG), and \citet{wei2023approximate} build directly on this framework. Our framework operates one stage earlier than these works. Rather than pruning or selecting arms for an online phase, it determines whether an online phase can resolve the ambiguity in the first place. Given a fixed action, it decomposes why an effect bound is wide, separating finite-sample noise from structural non-identifiability, and issues one of three recommendations: return the best action (\textsc{Return}), collect more samples of the same variables (\textsc{Collect}), or measure additional variables (\textsc{Observe}). Notably, {\citet{park2025structural}} identify finite-sample uncertainty in causal bounds as future work, which is precisely the source of uncertainty that our non-identifiability versus finite-sample decomposition addresses.

\section{Theoretical Analysis}
\label{thr-anl}

\begin{assumption}
    We assume the variables are discrete.
\end{assumption}
In the binary case, confidence intervals are defined over single conditional probabilities, which simplifies the exploration within the $\epsilon$-ball. For a non-binary variable $X$ with $|\mathcal{X}|$ states, there will be $|\mathcal{X}| - 1$ confidence intervals due to the simplex constraint. Constructing valid distributions from these confidence intervals becomes a difficult constrained optimization problem.

\begin{assumption}
    The structural causal model is Semi-Markovian.
\end{assumption}
\begin{assumption}
    We have access to the acyclic directed mixed graph (ADMG).
\end{assumption}

\subsection{Decision for Multiple Actions}

\begin{lemma}
    Let $L_i, U_i$ be the upper bound and lower bound of $\text{ATE}(x_0, x_i)$. Action $x_0$ is the best action if and only if $L_i \geq 0 \; \forall i\neq1$.
\end{lemma}

\begin{proof}
    $\implies$ is clear from the definition.

$\impliedby$: Suppose for the sake of contradiction that there exists some SCM such that $\text{ATE}(x_0, x_j) < 0$.  But this contradicts the assumption that $L_i \geq 0 \;\forall i\neq 1$.
\end{proof}

\begin{lemma}
    If there exists any $x_i$ such that $U_i >0$ and $L_i<0$, then both $x_0$ and $x_i$ cannot be the best action.
\end{lemma}

\begin{proof}
    If there exists any xi such that $U_i >0$ and $L_i<0$, that means there exist two SCM $S1$, $S2$ such that $\text{ATE}(x_0, x_i) >0$ in $S1$, and $\text{ATE}(x_0, x_i) < 0$ in $S2$. So both $x_0$ and $x_i$ cannot be determined to be the best action.

\end{proof}

To prove our main theorem, we need the following lemma.

\begin{lemma}
    \label{lm:connected}
    Let $V_1,\dots, V_n$ be the set of observed variables in the topological order. 

Suppose the observed variables can be partitioned into groups 
$ D_1,\dots, D_m $, such that variables within each group may share one common exogenous noise, while exogenous noises associated with distinct groups are mutually independent.

For any compatible observational distribution P, the SCMs entailing P admit a response function representation whose feasible parameter set is nonempty, compact, and convex.
\end{lemma}

\begin{proof}
    For each variable $V_i$, let $r_i$ be a deterministic function assigning a value of $V_i$ to each configuration of its observed parents. 

    For group $D_j$, let $R_{D_j}=(R_i:i\in D_j)$ be its collection of response variables with joint distribution $q_j(r_{D_j}) = P(R_{D_j}=r_{D_j}).$

    Since the exogenous noise are mutually independent across groups, the response variable vectors are independent
    $$ P(r_{D_1}, r_{D_2},\dots, r_{D_m}) = \prod_{j=1}^m q_j(r_{D_j}). $$
    
 For a fixed $v_{1:k}=(v_1,\dots,v_k)$, define
    \[
        H_{j,k}(v_{1:k};q_j)
        :=
        \sum_{r_{D_j}}
        q_j(r_{D_j})
        \prod_{\substack{i\le k\\ i\in D_j}}
        \mathbf 1
        \left\{
            r_i(v_{pa(i)})=v_i
        \right\}.
    \]
    By construction, $H_{j,k}(v_{1:k};q_j)$ is linear in $q_j$.

    Since $V_1,\dots,V_n$ are in topological order, all parents of $V_i$
    occur before $V_i$. Therefore, the event
    $V_{1:k}=v_{1:k}$ occurs if and only if, for every group $D_j$, all
    response functions in that group corresponding to variables among
    $V_1,\dots,V_k$ agree with the configuration $v_{1:k}$.
    Hence, by independence across groups,
\begin{equation}
    \label{eq:prefix-factorization}
            P_q(v_{1:k})
        =
        \prod_{j=1}^m H_{j,k}(v_{1:k};q_j),
\end{equation}
    where $P_q$ denotes the observational distribution induced by
    $q=(q_1,\dots,q_m)$.

    Let $d(k)$ denote the unique group containing $V_k$. When we extend from $v_{1:k-1}$ to $v_{1:k}$, only the factor corresponding to
    $D_{d(k)}$ changes. Thus, whenever
    $P(v_{1:k-1})>0$, requiring $P_q=P$ implies
    \[
        \frac{
            H_{d(k),k}(v_{1:k};q_{d(k)})
        }{
            H_{d(k),k-1}(v_{1:k-1};q_{d(k)})
        }
        =
        P(v_k\mid v_{1:k-1}),
    \]
    or equivalently,
    
    \begin{equation}
        \label{eq:linear-fixed-p}
        H_{d(k),k}(v_{1:k};q_{d(k)})
        =
        P(v_k\mid v_{1:k-1})
        H_{d(k),k-1}(v_{1:k-1};q_{d(k)}).
    \end{equation}

    Since $P(v_k\mid v_{1:k-1})$ is fixed, and both $H$ terms are linear in
    $q_{d(k)}$, each constraint in \eqref{eq:linear-fixed-p} is linear in a
    single group distribution.

    We now show that these constraints are also sufficient.
    Suppose that \eqref{eq:linear-fixed-p} holds for every $k$ and every
    prefix with $P(v_{1:k-1})>0$. We prove by induction on $k$ that
    \[
        P_q(v_{1:k})=P(v_{1:k}).
    \]
    The claim is trivial for the empty prefix. Assume it holds for
    $v_{1:k-1}$.

    If $P(v_{1:k-1})>0$, then by
    \eqref{eq:prefix-factorization} and
    \eqref{eq:linear-fixed-p},
    \[
    \begin{aligned}
        P_q(v_{1:k})
        &=
        P_q(v_{1:k-1})
        \frac{
            H_{d(k),k}(v_{1:k};q_{d(k)})
        }{
            H_{d(k),k-1}(v_{1:k-1};q_{d(k)})
        }\\
        &=
        P(v_{1:k-1})
        P(v_k\mid v_{1:k-1})\\
        &=
        P(v_{1:k}).
    \end{aligned}
    \]
    If $P(v_{1:k-1})=0$, then by the induction hypothesis
    $P_q(v_{1:k-1})=0$, and hence every extension $v_{1:k}$ also has
    probability zero under both $P_q$ and $P$.
    Therefore $P_q=P$.

    For each group $D_j$, let
    \[
        \mathcal Q_j(P)
        :=
        \left\{
            q_j:
            q_j(r_{D_j})\ge 0,\;
            \sum_{r_{D_j}}q_j(r_{D_j})=1,\;
            q_j \text{ satisfies all constraints
            \eqref{eq:linear-fixed-p} involving }D_j
        \right\}.
    \]
    Each $\mathcal Q_j(P)$ is the intersection of a finite-dimensional
    probability simplex with finitely many linear equalities. Hence
    $\mathcal Q_j(P)$ is a compact convex polytope.

    Moreover, the preceding necessity and sufficiency argument shows that
    the full feasible response-function parameter set is
    \[
        \mathcal Q(P)
        =
        \prod_{j=1}^m \mathcal Q_j(P).
    \]
    Since $P$ is assumed compatible with the SCM class,
    $\mathcal Q(P)$ is nonempty. A finite Cartesian product of compact
    convex sets is compact and convex. Therefore $\mathcal Q(P)$ is
    nonempty, compact, and convex.
    
\end{proof}

\textbf{Proof of Theorem~\ref{thm:id-surj}}

\idSurj*

\begin{proof}
    Consider any connected set of observed distributions $\mathcal{P}$. For any $P\in \mathcal{P}$, let $\bar{P}\in \bar{\Delta}$ be the distribution including the latent variables $U$, such that $\sum_U \bar{P} = P$. Assume the causal query is well defined, i.e. $P(pa(y)) > 0$.

    Let $adj(X,Y)$ be the set of variables that are the parents of $Y$ but not descendants of $X$. Then 
\begin{align*}
    h(\bar{P},G) &= P(y|{do}(x)) \\
    &= \sum_{adj(X,Y)} \bar{P}(y|x,adj(X,Y)) \bar{P}(adj(X,Y)) \\
    &= \frac{\sum_{v\not\in \{x,y,adj(X,Y)\}}\bar{P}(v)}{\sum_{v\not\in \{x,adj(X,Y)\}}\bar{P}(v)} \left(\sum_{v\not\in\{adj(X,Y)\}}\bar{P}(v)\right).
\end{align*}

Since the causal effect is a rational function of $\bar{P}$, it is continuous when the denominator is strictly positive. 

Let $f(P,G)$ be the function that maps $P\in\mathcal{P}$ to some causal effect in $[0,1]$ such that $f(P,G) = h(\bar{P},G)$ for all  $\bar{P}$ such that $\sum_u\bar{P} = P$. Since the query is identifiable, the value of $h(\bar{P},G)$ is identical for all $\{\bar{P}\mid \sum_U \bar{P} = P, P\in\mathcal{P}\}$. So $f $ is well defined, and $h = f \circ g$.

Now for any open subset $M\subseteq [0,1]$, since $h$ is continuous, the preimage $h^{-1}(M)\subseteq \{\bar{P}\mid \sum_U \bar{P} = P, P\in\mathcal{P}\}$ is open. Let $g$ be the function mapping full joint distribution to observed distribution, i.e., $g:\bar{P}\mapsto P$. By our construction, for any $P\in \mathcal{P}$, $g:h^{-1}(M)\rightarrow P$ is a linear surjective function. With the relative topologies on the simplices, $g$ is an open map. Since $h^{-1}(M)$ is open, $g(h^{-1}(M))$ is open. For any open set $M\subseteq [0,1]$, $f^{-1}(M) = g(h^{-1}(M))$ is open, so the function $f$ is continuous. 

By the generalized intermediate value theorem, the image of $f(\mathcal{P},G)$ is connected. Since the image is a subset of $[0,1]$, it is an interval that contains $[a,b]$.

\end{proof}

\textbf{Proof of \Cref{thm:nonid-surj}}
\label{ap:thm3}

\nonidsurj*

\begin{proof}
Suppose the causal query of interest is $P(y|{do}(x))$. Consider the graph $G$ that includes both observed and unobserved latent variables. Let ${\mathcal{{P}}}$ be the set of distributions over the graph $G$. 
For each $\bar{P}\in\mathcal{P}$, let the set of $\bar{P}$ denote the underlying SCM, i.e. $S= \bar{P}$ for some $\bar{P}\in\mathcal{P}$. Assume the causal query is well defined, i.e. $P(pa(y)) > 0$. Then 
$$f(p,S) = P(y|{do}(x)) = \sum_{pa(y)} \bar{P}(y|pa(y)) \bar{P}(pa(y)) = \frac{\sum_{v\not\in \{x,y,pa(y)\}}\bar{P}(v)}{\sum_{v\not\in \{x,pa(y)\}}\bar{P}(v)} \left(\sum_{v\not\in\{pa(y)\}}\bar{P}(v)\right).$$

So the causal effect is a rational function of the full distribution over $G$. Since the causal query is well defined, the denominator is strictly positive. Thus $f(p,S) $ is continuous over all $\bar{P}\in \mathcal{P}$.
For each $\bar{P}\in \mathcal{P}$, the only set of constraints are in form of $\sum_{u\in U} \bar{P}(v,u) = P(v)$. The set $\mathcal{P}$ is connected. Therefore, the image of $f(p,\cdot)$ is connected, and thus it is an interval. 
Since there is exist $f(p,S)  = a$ and $f(p,S)  = b$, and the image of $f(p,S) $ is connected, $[a,b] \subseteq f(p,\mathcal{S})$. so the map $f(p,\cdot) $ is surjective.

\end{proof}

\textbf{Proof of Proposition \ref{prop:irred-samp-ate}}

\irred*

\begin{proof}
    Recall the \cref{def:uncertainty-mult}, $\mathcal{I}_P = \{ATE(x)\mid S\in \mathcal{S}_P\}$, and $\mathcal{S}_P = \{S \text{ an SCM} \mid P_{obs}(S)=P, S\entails G\}$.
    Suppose the coverage event $P^\star \in \mathcal{C}(D)$ holds, i.e., $\{P^\star\} \subseteq \mathcal{C}(D)$.

    For (i), since an intersection over a larger collection of sets is smaller, we have $\bigcap_{P\in\mathcal{C}(D)} \mathcal{I}_P(x) \subseteq \bigcap_{P\in\{P^\star\}} \mathcal{I}_P(x) = \mathcal{I}_{P^\star}(x)$.

    For (ii), if $L_{\mathrm{inner}} < 0 < U_{\mathrm{inner}}$, the inner region contains a value below zero and a value above zero. By (i), both values belong to $\mathcal{I}_{P^\star}(x)$, so $\mathcal{I}_{P^\star}(x)$ also crosses zero. Since the coverage event holds with probability at least $1-\alpha$, so does this conclusion.
   
   Note that, unlike a comparison between confidence sets obtained from two datasets, this argument does not require their empirical distributions to be identical or the confidence sets to be nested.
\end{proof}

\textbf{On the arrival of new datasets:} We emphasize that the conclusion in Proposition \ref{prop:irred-samp-ate} concerns the fixed population quantity $\mathcal{I}_{P^\star}(x)$, not the data-dependent inner region: the inner region computed from a future dataset need not be monotone in the sample size, and need not cross zero again, but this does not change $\mathcal{I}_{P^\star}(x)$. Consequently, the argument does not require future empirical distributions or confidence sets to be identical or nested.

\textbf{Proof of \Cref{thm:intersect-mult}}

\intersectmult*

\begin{proof}
For each $P\in \mathcal C$, write
\[
\mathcal I_P=[\ell(P),u(P)],
\]
where
\[
\ell(P)
:=
\min_{S:\,P_{\mathrm{obs}}(S)=P}
\operatorname{ATE}_S(x),
\qquad
u(P)
:=
\max_{S:\,P_{\mathrm{obs}}(S)=P}
\operatorname{ATE}_S(x).
\]
Assume that
\[
\bigcap_{P\in\mathcal C}\mathcal I_P
\]
is nonempty. Since the intersection of intervals is again an interval, we may write
\[
\bigcap_{P\in\mathcal C}\mathcal I_P
=
[\underline c,\overline c].
\]

We first show that
\[
\underline c=\sup_{P\in\mathcal C}\ell(P).
\]
Let
\[
\bar L:=\sup_{P\in\mathcal C}\ell(P).
\]

Suppose first, for contradiction, that $\bar L<\underline c$. Choose
\[
\theta\in(\bar L,\underline c).
\]
Since $\theta<\underline c$, we have
\[
\theta\notin \bigcap_{P\in\mathcal C}\mathcal I_P.
\]
Therefore, there exists some $P_0\in\mathcal C$ such that
\[
\theta\notin \mathcal I_{P_0}=[\ell(P_0),u(P_0)].
\]
Hence either $\theta<\ell(P_0)$ or $u(P_0)<\theta$.

If $\theta<\ell(P_0)$, then
\[
\bar L=\sup_{P\in\mathcal C}\ell(P)\ge \ell(P_0)>\theta,
\]
contradicting $\theta>\bar L$.

If $u(P_0)<\theta$, then since $\theta<\underline c$, we have
\[
u(P_0)<\underline c.
\]
Thus $\underline c\notin [\ell(P_0),u(P_0)]$, contradicting the fact that
\[
\underline c\in \bigcap_{P\in\mathcal C}\mathcal I_P.
\]
Therefore, $\bar L\ge \underline c$.

Conversely, since $\underline c$ belongs to every interval $\mathcal I_P$, we have
\[
\ell(P)\le \underline c
\qquad
\text{for all } P\in\mathcal C.
\]
Taking the supremum over $P\in\mathcal C$ gives
\[
\bar L=\sup_{P\in\mathcal C}\ell(P)\le \underline c.
\]
Therefore,
\[
\underline c=\sup_{P\in\mathcal C}\ell(P).
\]

Similarly, we show that
\[
\overline c=\inf_{P\in\mathcal C}u(P).
\]
Let
\[
\bar U:=\inf_{P\in\mathcal C}u(P).
\]
Since $\overline c$ belongs to every interval $\mathcal I_P$, we have
\[
\overline c\le u(P)
\qquad
\text{for all } P\in\mathcal C.
\]
Thus
\[
\overline c\le \inf_{P\in\mathcal C}u(P)=\bar U.
\]

Now suppose, for contradiction, that $\overline c<\bar U$. Choose
\[
\theta\in(\overline c,\bar U).
\]
Since $\theta>\overline c$, we have
\[
\theta\notin \bigcap_{P\in\mathcal C}\mathcal I_P.
\]
Therefore, there exists some $P_0\in\mathcal C$ such that
\[
\theta\notin \mathcal I_{P_0}=[\ell(P_0),u(P_0)].
\]
Hence either $\theta<\ell(P_0)$ or $u(P_0)<\theta$.

If $\theta<\ell(P_0)$, then since $\overline c<\theta$, we have
\[
\overline c<\ell(P_0),
\]
so $\overline c\notin [\ell(P_0),u(P_0)]$, contradicting
\[
\overline c\in\bigcap_{P\in\mathcal C}\mathcal I_P.
\]

If $u(P_0)<\theta$, then
\[
\bar U=\inf_{P\in\mathcal C}u(P)\le u(P_0)<\theta,
\]
contradicting $\theta<\bar U$.

Therefore,
\[
\overline c=\inf_{P\in\mathcal C}u(P).
\]

Combining the two endpoint identities gives
\[
\bigcap_{P\in\mathcal C}\mathcal I_P
=
\left[
\sup_{P\in\mathcal C}\ell(P),
\inf_{P\in\mathcal C}u(P)
\right].
\]
Substituting the definitions of $\ell(P)$ and $u(P)$, we obtain
\[
\bigcap_{P\in\mathcal C}\mathcal I_P
=
\left[
\sup_{P\in\mathcal C}
\min_{S:\,P_{\mathrm{obs}}(S)=P}
\operatorname{ATE}_S(x),
\;
\inf_{P\in\mathcal C}
\max_{S:\,P_{\mathrm{obs}}(S)=P}
\operatorname{ATE}_S(x)
\right].
\]
Therefore,
\[
\bigcap_{P\in\mathcal C}\mathcal I_P
=
[L_{\operatorname{ate}_x},U_{\operatorname{ate}_x}].
\]
If the extrema over $P\in\mathcal C$ are attained, the supremum and infimum above can equivalently be written as a maximum and a minimum, respectively.
\end{proof}

\textbf{Proof of \Cref{outer-region-mult}}

\outerregionmult*

\begin{proof}
    Since $\bigcap_{P\in\CR}\mathcal{I}_P = [\Lp_{ate_x},\Up_{ate_x}]$, we only need to show that $\bigcup_{P\in\CR} \mathcal{I}_P=[\Lo_{ate_x},\U_{ate_x}]$. Let $m_P = \min \mathcal{I}_P = \min_{S\in\{S \text{ an SCM} |P_{\text{obs}}(S)={P}\}} ATE(x)$ and $\mathcal{M} = \{m_P \mid P\in \CR\}$. Then $\min \bigcup_{P\in\CR}\mathcal{I}_P  =  \min \mathcal{M} = \min \min_{S\in\{S \text{ an SCM} |P_{\text{obs}}(S)={P}\}}ATE(x) =  \Lo_{ate_x}$. Similarly, by a symmetric argument, we have $\max \bigcup_{P\in\CR}\mathcal{I}_P =\U_{ate_x}$
\end{proof}

\textbf{Proof of \Cref{thm:decision-ate}}

\decisionate*

\begin{proof}
    By the definition of $ATE(x)$, for a fixed SCM, an action $x$ is the best action if $ATE(x)\geq 0$. For a fixed distribution $P$, $\mathcal{I}_P$ corresponds to the values of $ATE(x)$ with the set of SCMs that is compatible with $P$, where each SCM corresponds to a unique value of $ATE(x)$. And if $ATE(x)\geq 0$ for some $x$, the action $x$ is the optimal action.
    
    For a confidence set $\mathcal{C}$, if $\Latex > 0$, then for each $P\in \mathcal{C}$ and all the SCMs that are compatible with $P$, we have $ATE(x) \geq \Latex > 0$. Then all $\mathcal{I}_P$ lie above zero, and therefore $\bigcup_{P\in\mathcal{C}} \mathcal{I}_P=[\Latex,\Uatex]$ is above zero, i.e., $\Lo_{ate_x}\geq 0$, and the action $x$ is better than any other actions in all SCMs that are compatible with all distributions in the confidence set.
 
    Therefore, $x$ is the best action unambiguously. Similarly, if $\Uatex < 0$, then for all $P\in \mathcal{C}$ and all the SCMs that are compatible with $P$, we have $ATE(x)<0$, which indicates that $x$ is not the best action in any of these SCMs. If $\Lpatex<0<\Upatex$, by \Cref{thm:intersect-mult}, for all $P \in\mathcal{C}$, the $\mathcal{I}_P$ contains both positive and negative values, so there exist two SCMs $S_1$ and $S_2$ such that $x$ is the best action in $S_1$, but not the best action in $S_2$, and the decision is ambiguous. Moreover, on the coverage event $P^\star \in \mathcal{C}$, which holds with probability at least $1-\alpha$, this is true in particular for $P^\star$, i.e., $S_1$ and $S_2$ can be chosen to entail $P^\star$. Then $S_1$ and $S_2$ are compatible with any confidence set that contains $P^\star$, including one constructed from a larger dataset. Therefore, with probability at least $1-\alpha$, the best action remains ambiguous with more data.
\end{proof}

\subsection{Decisions for Single Action}
\label{sec:sing-act}

In many scenarios, instead of comparing effects of two or more actions, we need to compare the effect of a single action, $P(y|\Do(x))$ with a scalar quantity such as the marginal $P(y)$. 
Suppose we need to make the critical decision 
of choosing an action $X=x$ for the population level to improve an outcome.
We can decide that by comparing the bounds of $P(Y \mid do(x))$ with the observational probability $P(Y)$. Specifically, action $X=x$ is beneficial or harmful if  the interval $[\min P(Y \mid do(x)), \max P(Y \mid do(x))]$ is either completely above or below $P(Y)$. Otherwise, we need to collect or observe additional variables until we can make that decision. In this section, we consider all bounds to be constructed from $P(Y|do(x))$.

\textbf{Backdoor Example:}
\label{sec:backdoor_exp}
For example, let us consider the backdoor graph example as shown in the top row of Figure \ref{fig:bound-illustration}. Let $\hat{P}(y|x,z_i)$ and $\hat{P}(z)$ be the empirical estimation from the data and 
$l_i^y$ and $u_i^y$ be the confidence intervals of the conditional probabilities $\hat{P}(y|x,z_i)$. Similarly, $l_i^z$ and $u_i^z$ are the confidence intervals of the conditional probabilities $\hat{P}(z_i)$. The for $\hat{P}$ in the confidence set described by the above constraint, the maximum $P(y|\Do(x))$  can be estimated by the following (left):

\begin{minipage}[c]{1\textwidth}
\centering
\begin{minipage}[c]{0.45\linewidth}
\begin{align*}
    &\max P(y|{do(}x)) = \sum_i \hat{P}(y|x,z_i)\hat{P}(z_i)\\
    \text{s.t. }&\qquad \hat{P}(y|x,z_i)\in[l_i^y, u_i^y] ;\forall_i\\
    &\qquad \hat{P}(z_i) \in [l_i^z, u_i^z];\forall_i\\
    &\qquad \sum_i\hat{P}(z_i)=1
\end{align*}
\end{minipage}
\begin{minipage}[c]{0.05\textwidth}
\centering
$\quad\Longrightarrow\quad$
\end{minipage}
\begin{minipage}[c]{0.45\textwidth}
\begin{align*}
    &\max P(y|{do(}x)) = \sum_i u_i^y\hat{P}(z_i)\\
    \text{s.t. }
    &\qquad \hat{P}(z_i) \in [l_i^z, u_i^z];\forall_i\\
    &\qquad \sum_i\hat{P}(z_i)=1
\end{align*}
\end{minipage}
\end{minipage}

The objective function in the above optimization is a quadratic function, which could be nontrivial to solve. But in this case, since the $ a_i\in[l_i^y, u_i^y]$ are the only constraints for $a_i$, and the objective function is an affine function of $a_i$, the maximum is attained at the boundary when $a_i=u_i^y\quad \forall i $. The optimal solution of the above optimization problem can be found by the equivalent problem at the right, which can be solved through linear programming. 

Below, we specify the uncertainty for any specific action and the decision-making based on that. 

\begin{definition}[\epSample and \epNonID Uncertainty (Figure~\ref{fig:bound-illustration})]
\label{def:uncertainty}
    Let an empirical estimation of the joint distribution $P$ be $\hat{P}$ and a confidence set $\CR$ be such that $\CR$ covers the true $P$ with probability of at least $1-\alpha$. Define the set $\mathcal{S}_P=\{S\text{ an SCM} \mid P_\text{obs}(S)=P,  {S \models G} \}$ and $\mathcal{I}_P = \{P_S(y|{do(}x)) \mid S\in\mathcal{S}_P\}$. 
    Then define the intersection $\bigcap_{{P}\in\CR} \mathcal{I}_{{P}}$ to be the \textbf{\epNonID uncertainty}, and the set difference $\bigcup_{{{P}}\in\CR} \mathcal{I}_{{P}}\setminus \bigcap_{{{P}}\in\CR} \mathcal{I}_{{P}}$ to be the \textbf{\epSample uncertainty}.
    {The epistemic uncertainty region in the estimation would be the union of these subcomponents.  {The components are named the inner region and the outer band, respectively.} }
\end{definition}

\begin{proposition}
\label{prop:irred-samp}
    For a causal decision-making problem with $P(y|\Do(x))$ as the decision metric for a single action $x$, let $P^\star$ be the true observational distribution, $\mathcal{C}(D)$ a confidence set satisfying $\Pr\left(P^\star \in \mathcal{C}(D)\right) \geq 1-\alpha$, and $\mathcal{I}_{P}(x)$ the set of values of $P(y|\Do(x))$ identified for observational distribution $P$. Then, with probability at least $1-\alpha$, the following hold:
    \begin{enumerate}[label=(\roman*)]
        \item $\bigcap_{P \in \mathcal{C}(D)} \mathcal{I}_{P}(x) \subseteq \mathcal{I}_{P^\star}(x)$.
        \item If the inner region of $P(y|\Do(x))$ crosses the decision boundary, i.e., $\Lpx < \gamma < \Upx$, then the population identified set $\mathcal{I}_{P^\star}(x)$ also crosses $\gamma$.
    \end{enumerate}
\end{proposition}

\begin{definition}[Inner region: $\inner$, Outer band $\outr$]
\label{def:four}

Given a confidence set $\CR$ of an empirical distribution $\hat{P}$, Define the set $\mathcal{S}_P=\{S\text{ an SCM} \mid P_\text{obs}(S)=P,  {S \models G} \}$. We access the sample and non-id uncertainty regions by estimating the following quantities: %
\begin{equation}
\label{eq:4quant}
\begin{split}
\Ux:&=\max_{{P} \in \CR(\hat{{P}})} \max_{\text{ }S\in\mathcal{S}_P}P_{S}(Y|{do(}X=x))\\
\Upx:&=\min_{{{P}} \in \CR(\hat{{P}})} \max_{\text{ }S\in\mathcal{S}_P}P_{S}(Y|{do(}X=x))\\
\Lpx:&=\max_{{{P}} \in \CR(\hat{{P}})} \min_{\text{ }S\in\mathcal{S}_P}P_{S}(Y|{do(}X=x))\\
\Lox:&=\min_{{{P}} \in \CR(\hat{{P}})} \min_{\text{ }S\in\mathcal{S}_P}P_{S}(Y|{do(}X=x)),
\end{split}
\end{equation}
\end{definition}

\begin{restatable}{proposition}{redsamp}
    \label{prop:red-samp}
    There exists SCMs $S\in \mathcal{S}_P$ where the uncertainty in the causal decision-making problem
     due to  \epSample uncertainty in $P(y|\Do(x))$ (i.e., outer band) 
     can be reduced by increasing the number of samples in the data.
\end{restatable}

\begin{proof}
    By the definition, the \epSample uncertainty are the intervals $[\Up_x, \U_x]$ and $[\Lo_x, \Lp_x]$. As the number of samples increases, the radius $r_n = c/\sqrt{n}$ of the epsilon ball that is searched over decreases. Since $\U_x$ and $\Up_x$ are the supremum and infimum of $u(P) = \max_{S\in\mathcal{S}_P} P_S(y|\Do(x))$ over $\CR(\hat{P})$, the width of $[\Up_x, \U_x]$ is at most the variation of $u$ over a ball of radius $r_n$, which becomes narrower whenever $u$ varies continuously in $P$. For example, in the graph $X \to Y$ without unobserved confounding, $u(P) = P(y|x)$ and the width is exactly $2r_n$. Similarly for $[\Lo_x, \Lp_x]$.
\end{proof}

\begin{theorem}
\label{thm:intersect}
    Given a confidence set, $\CR$ and a set of causal effects, $\mathcal{I}_P$ (Def.~\ref{def:uncertainty}), $\bigcap_{P\in\CR}\mathcal{I}_P = [\Lpx,\Upx]$ if not empty.
\end{theorem}
\begin{proof}
    Let $m_P = \min \mathcal{I}_P = \min_{S\in\{S \text{ an SCM} |P_{\text{obs}}(S)={P}\}}P_{S}(Y|{do(}X=x))$ and $\mathcal{M} = \{m_P \mid P\in \CR\}$. Then $\min \bigcap_{P\in\CR}\mathcal{I}_P  =  \max \mathcal{M} = \max \min_{S\in\{S \text{ an SCM} |P_{\text{obs}}(S)={P}\}}P_{S}(Y|{do(}X=x)) =  \Lpx$.

    By a symmetric argument, we can show that $\max \bigcap_{P\in\CR}\mathcal{I}_P  =  \max \min_{S\in\{S \text{ an SCM} |P_{\text{obs}}(S)={P}\}}P_{S}(Y|{do(}X=x)) = \Upx$.
\end{proof}

\begin{corollary}
\label{outer-region}
     Given a confidence set, $\CR$ and a set of causal effects, $\mathcal{I}_P$ (Def.~\ref{def:uncertainty}), we have \\
     $\bigcup_{P\in\CR} \mathcal{I}_P\setminus \bigcap_{P\in\CR} \mathcal{I}_P=[\Lox,\Lpx]\cup[\Upx,\Ux]$
\end{corollary}
\begin{proof}
    Since $\bigcap_{P\in\CR}\mathcal{I}_P = [\Lpx,\Upx]$, we only need to show that $\bigcup_{P\in\CR} \mathcal{I}_P=[\Lox,\Ux]$. Let $m_P = \min \mathcal{I}_P = \min_{S\in\{S \text{ an SCM} |P_{\text{obs}}(S)={P}\}}P_{S}(Y|{do(}X=x))$ and $\mathcal{M} = \{m_P \mid P\in \CR\}$. Then $\min \bigcup_{P\in\CR}\mathcal{I}_P  =  \min \mathcal{M} = \min \min_{S\in\{S \text{ an SCM} |P_{\text{obs}}(S)={P}\}}P_{S}(Y|{do(}X=x)) =  \Lox$. Similarly, by a symmetric argument, we have $\max \bigcup_{P\in\CR}\mathcal{I}_P =\Ux$
\end{proof}

\begin{definition}
    An causal decision with  $P(y|\Do(x))$ as decision estimand is ambiguous
    given confidence set $\CR$ if there exists two SCMs $S_1,S_2$ where $p_{S_1}(y_1|{do(}x))>{P(y)}$ and $p_{S_2}(y_1|{do(}x))<{P(y)}$ where $S_1, S_2$ entail 
    $p_1,p_2\in \CR$.
    
\end{definition}

\begin{theorem}
\label{thm:decision}
    For a causal decision making problem with $P(y|\Do(x))$ bound, let $(\Ux, \Upx, \Lpx, \Lox)$ be estimated from the data, and let $\gamma = P(y)$ be the decision boundary. The decision of $X=x$ as the best action (or cannot be the best action) is unambiguous if $\Lox > \gamma$ (or $\gamma > \Ux$). The decision is ambiguous and, with probability at least $1-\alpha$, cannot be improved with more data if $\Lpx < \gamma < \Upx$.
\end{theorem}

\begin{proof}
    The proof follows that of \Cref{thm:decision-ate}, replacing $ATE(x)$ with $P(y|\Do(x))$, the decision boundary $0$ with $\gamma$, the quantities $(\Uatex, \Upatex, \Lpatex, \Latex)$ with $(\Ux, \Upx, \Lpx, \Lox)$, \Cref{thm:intersect-mult} with \Cref{thm:intersect}, and \Cref{prop:irred-samp-ate} with \Cref{prop:irred-samp}.
\end{proof}

Note that $P(y)$ can be provided by a domain expert. If we estimate it from data, the framework for multiple actions discussed in the previous section can be utilized to make a decision in such cases.

\subsection{Approximation with $\epsilon$-net}

The following lemma applies even with unobserved confounding between $X,Y$. The corollary shows that by searching over small enough epsilon net, the output value should find close enough to the true optimal.

\begin{lemma}
    Consider the causal graph with binary $X,Y,Z$ where $Z$ is an observed confounder with observational distribution $P(X,Y,Z)$. Consider another joint distribution $P'(X,Y,Z)$. If $TVD(P(X,Y|Z=z), P'(X,Y|Z=z))<\epsilon$ and $TVD(P(Z), P'(Z))<\epsilon$ then $|f(p) - f(p')|\leq 3\epsilon$ where $f(p)=\max_{\{S\mid P_\text{obs}(S)=p\}} p_S(y|\Do(x))$. 
\end{lemma}

\begin{proof}
    Consider the bow-backdoor graph with binary variables $X,Y,Z$ where $Z$ is a confounder. Define the response variables $R_y, R_x$ as follows.

\begin{table}[h]
    \centering
    \begin{tabular}{|c|c c|}
        \hline
        $R_x$ & $z_0$ & $z_1$ \\
        \hline
        $0$   & $y_0$   & $y_0$   \\
        $1$   & $y_0$   & $y_1$   \\
        $2$   & $y_1$   & $y_0$   \\
        $3$   & $y_1$   & $y_1$   \\
        \bottomrule
    \end{tabular}
    \caption{Response function $R_x$}
    \label{tab:rx}
\end{table}

\begin{table}[h]
    \centering
    \begin{tabular}{|l|cccc|}
        \hline
        $R_y$& $z_0, x_0$ & $z_0, x_1$ & $z_1, x_0$ & $z_1, x_1$ \\
        \hline
        $0$ & $y_0$ & $y_0$ & $y_0$ & $y_0$ \\
        $1$ & $y_0$ & $y_0$ & $y_0$ &$y_1$\\
        $2$ & $y_0$ & $y_0$ & $y_1$ & $y_0$ \\
        $3$ & $y_0$ & $y_0$ & $y_1$ & $y_1$ \\
        $4$ & $y_0$ & $y_1$ & $y_0$ & $y_0$ \\
        $5$ & $y_0$ & $y_1$ & $y_0$ & $y_1$ \\
        $6$ & $y_0$ & $y_1$ & $y_1$ & $y_0$ \\
        $7$ & $y_0$ & $y_1$ & $y_1$ & $y_1$ \\
        $8$ & $y_1$ & $y_0$ & $y_0$ & $y_0$ \\
        $9$ & $y_1$ & $y_0$ & $y_0$ & $y_1$ \\
        $10$ & $y_1$ & $y_0$ & $y_1$ & $y_0$ \\
        $11$ & $y_1$ & $y_0$ & $y_1$ & $y_1$ \\
        $12$ & $y_1$ & $y_1$ & $y_0$ & $y_0$ \\
        $13$ & $y_1$ & $y_1$ & $y_0$ & $y_1$ \\
        $14$ & $y_1$ & $y_1$ & $y_1$ & $y_0$ \\
        $15$ & $y_1$ & $y_1$ & $y_1$ & $y_1$ \\
        \bottomrule
    \end{tabular}
    \caption{Response function $R_y$}
    \label{tab:ry}
\end{table}

Let $f(p) = \max p(y_0|{do}(x_0))$ then we can express $f(p)$ with response variables as following

\begin{equation}
    \begin{aligned}
         f(p) &= \sum_{j} \left(\left(\sum_{i=0}^7 q_{i,j}p(z_0)\right) + \left(\sum_{i=\{0,1,4,5,8,9,12,13\}}q_{i,j}p(z_1)\right)\right)\\
    \end{aligned}
\end{equation}
where $q_{i,j}= p(R_x=i, R_y=j) $

\begin{equation*}
    \begin{aligned}
    \mathbf{A}= \begin{bmatrix}
    1 \\ 
    1 \\ 
    1 \\ 
    1 \\ 
    p(z_0)\\
    p(z_0)\\
    p(z_0)\\
    p(z_0)\\
    p(z_1)\\
    p(z_1)\\
    p(z_1)\\
    p(z_1)\\
    1\\
    \vdots \\
\end{bmatrix}, && 
\mathbf{v} = \begin{bmatrix}
   q_{0,0} \\
   q_{0,1} \\
   q_{0,4} \\
   q_{0,5} \\
   q_{0,2} \\
   q_{0,3} \\
   q_{0,6} \\
   q_{0,7} \\
    q_{0,8} \\
   q_{0,9} \\
   q_{0,12} \\
   q_{0,13} \\
   q_{1,0}\\
   \vdots \\
\end{bmatrix}
\end{aligned}
\end{equation*}

\begin{equation*}
    \begin{aligned}
\mathbf{v}_{z_0} = \begin{bmatrix}
   q_{0,0} \\
   q_{0,1} \\
   q_{0,2} \\
   q_{0,3} \\
   q_{0,4} \\
   q_{0,5} \\
   q_{0,6} \\
   q_{0,7} \\
    q_{1,0} \\
   \vdots \\
\end{bmatrix}, && 
\mathbf{v}_{z_1} = \begin{bmatrix}
   q_{0,0} \\
   q_{0,1} \\
   q_{0,4} \\
   q_{0,5} \\
   q_{0,8} \\
   q_{0,9} \\
   q_{0,12} \\
   q_{0,13} \\
    q_{1,0} \\
   \vdots \\
\end{bmatrix}
\end{aligned}
\end{equation*}

The domain of $\mathbf{x}$ is defined by the following constraints

\begin{align*}
    p(y_0, x_0|z_0)  &= q_{0,0} + q_{0,1} + q_{0,2} + q_{0,3} + q_{0,4} + q_{0,5} + q_{0,6} + q_{0,7} \\
    &+ q_{1,0} + q_{1,1} + q_{1,2} + q_{1,3} + q_{1,4} + q_{1,5} + q_{1,6} + q_{1,7} \\
    p(y_0, x_0|z_1)  &= q_{0,0} + q_{0,1} + q_{0,4} + q_{0,5} + q_{0,8} + q_{0,9} + q_{0,12} + q_{0,13}\\
    &+ q_{2,0} + q_{2,1} + q_{2,4} + q_{2,5} + q_{2,8} + q_{2,9} + q_{2,12} + q_{2,13}\\
    \dots \\
    p(y_1,x_1|z_1) &= ...
\end{align*}

Assume $TVD(P(X,Y|z),P'(X,Y|z))\leq \epsilon$ and $TVD(P(Z),P'(Z))\leq \epsilon$ for each $z$ .

Let $k_{xyz}=\{i,j| \sum q_{i,j}= p(x,y|z)\}$. From the construction, for each $z$, $k_{xyz}$ are non-overlap for each $Z=z$, and $\sum_{xy}k_{xyz} = 1$. 

Since $TVD(P(Z),P'(Z))\leq \epsilon$, $|p(z)-p'(z)|\leq \epsilon$ for all $z$.

Let $S_z = \{x,y,z| k_{xyz}\cap \mathbf{v}_z\neq\emptyset\}$
Then we have 
\begin{equation}
       \begin{aligned}
        \frac{1}{2}\sum_{x,y,z\in S_z} \Big| (p(x,y|z)-p'(x,y|z))\Big |&\leq \frac{\epsilon}{2}\\
       \sum_{x,y,z\in S_z} \Big | \sum_{i,j\in k_{xyz}}(q_{i,j} - q'_{i,j}) \Big | &\leq \epsilon\\
        | \sum\mathbf{v}_z- \sum\mathbf{v}'_z|&\leq \epsilon \\
    \end{aligned} 
\end{equation}

Hence

\begin{equation}
    \begin{aligned}
        |f(p)-f(p')| & = \left|\left(p(z_0)\sum\mathbf{v}_{z_0} +p(z_1)\sum\mathbf{v}_{z_1}\right) - \left(p'(z_0)\sum\mathbf{v}_{z_0}' +p'(z_1)\sum\mathbf{v}_{z_1}'\right)\right| \\
        &= \left|\left(p(z_0)\sum\mathbf{v}_{z_0} -p'(z_0)\sum\mathbf{v}_{z_0}'\right) + \left( p(z_1)\sum\mathbf{v}_{z_1} -p'(z_1)\sum\mathbf{v}_{z_1}'  \right)\right| \\
        &\leq \left|\left(p(z_0)\sum\mathbf{v}_{z_0} -p'(z_0)\sum\mathbf{v}_{z_0}'\right)\right| + \left|\left( p(z_1)\sum\mathbf{v}_{z_1} -p'(z_1)\sum\mathbf{v}_{z_1}'  \right)\right| \\
        &= \left|\left(p(z_0)\sum\mathbf{v}_{z_0} -p(z_0)\sum\mathbf{v}_{z_0}' +p(z_0)\sum\mathbf{v}_{z_0}'-p'(z_0)\sum\mathbf{v}_{z_0}'\right)\right| \\
        &\qquad+ \left|\left( p(z_1)\sum\mathbf{v}_{z_1} -p(z_1)\sum\mathbf{v}_{z_1}'+p(z_1)\sum\mathbf{v}_{z_1}'-p'(z_1)\sum\mathbf{v}_{z_1}'  \right)\right| \\
        &= \left|p(z_0)(\sum\mathbf{v}_{z_0}- \sum\mathbf{v}'_{z_0}) + \sum v'_{z_0} (p(z_0)-p'(z_0))\right| \\
        &\qquad +\left|p(z_1)(\sum\mathbf{v}_{z_1}- \sum\mathbf{v}'_{z_1}) + \sum v'_{z_1} (p(z_1)-p'(z_1))\right| \\
        & \leq p(z_0)\epsilon + \sum v'_{z_0} \epsilon + p(z_1)\epsilon + \sum v'_{z_1} \epsilon \\
        &\leq 3\epsilon
    \end{aligned}
\end{equation}

\end{proof}

\begin{corollary}
    For binary variables $X,Y$ with a binary confounder $Z$ and latent confounders. The $\Up_x$ and $\Lp_x$ estimated from our algorithm are within $3\epsilon$ distance from the true values. 
\end{corollary}

\begin{proof}
    Let $\Up_x$ denote the true $\min\max P(y|{do}(x))$ and $\hat{\Up}_x$ denote the estimated value of our algorithm. Suppose for the sake of contradiction that $|\hat{\Up}_x-\Up_x|>3\epsilon$. By Lemma 1, the observed distributions that are compatible with $\Up_x$ and $\hat{\Up}_x$ do not belong to the same ball with $\epsilon$ distance. Let $B$ be the ball that contains the observed distribution that is compatible with $\Up_x$, and $U_x$ be the maximum value of $P(y|{do}(x))$ with a compatible observational distribution inside $B$. By Lemma 1, $|U-\Up_x|\leq 3\epsilon$. Since $\Up_x, \hat{\Up}_x, U$ are scalers and $\Up_x$ is the smallest $\max P(y|{do}(x))$ over all observational distributions, we have $U-\Up_x\leq 3\epsilon$, $\hat{\Up}_x - \Up_x>3\epsilon$. And that implies $\hat{\Up}_x > U$, which contradicts the assumption that $\hat{\Up}_x$ is the smallest $\max P(y|{do}(x))$ among all epsilon balls.
    So we can conclude that $|\hat{\Up}_x-\Up_x|\leq 3\epsilon$
\end{proof}

\begin{lemma}
     For an observational distribution from IV graph $P(X,Y,Z)$ where $X,Y,Z$ are binary variables, there exist a SCM such that $P(y_1|{do}(x_1))$ attains the upper bound and $P(y_1|{do}(x_0))$ attains the lower bound. 
     
\end{lemma}

\begin{proof}
    Let $R_x, R_y$ be the response variables for $X,Y$, and $q_{ij}:= P(R_x=i, R_y=j)$. Then we have the following constraints from the observed data.
    \begin{align*}
        P(y_0,x_0|z_0) &= q_{00} + q_{10} + q_{01} + q_{11} \\
        P(y_0,x_1|z_0) &= q_{20} + q_{22} + q_{30} + q_{32} \\
        P(y_1,x_0|z_0) &= q_{02} + q_{03} + q_{12} + q_{13} \\
        P(y_1,x_1|z_0) &= q_{21} + q_{23} + q_{31} + q_{33} \\
        P(y_0,x_0|z_1) &= q_{00} + q_{20} + q_{01} + q_{21} \\
        P(y_0,x_1|z_1) &= q_{10} + q_{12} + q_{30} + q_{32} \\
        P(y_1,x_0|z_1) &= q_{02} + q_{03} + q_{22} + q_{23} \\
        P(y_1,x_1|z_1) &= q_{11} + q_{13} + q_{31} + q_{33} \\
    \end{align*}

\begin{align*}
    P(y_1|{do}(x_1)) &= \sum_{i=0}^3 q_{i1} + q_{i3} \\
    P(y_1|{do}(x_0)) &= \sum_{i=0}^3 q_{i2} + q_{i3} \\
\end{align*}

The goal is to show an SCM that obtains the minimum of $P(y_1|{do}(x_0))$ and the maximum of $P(y_1|{do}(x_1))$. 
For the SCM minimize $P(y_1|{do}(x_0))$, the corresponding $q_{ij}$ are minimized. Without loss of generality, assume $P(y_1,x_0|z_1)>P(y_1,x_0|z_0)$. The minimum of $P(y_1|{do}(x_0))$ can be obtained by let $q_{12}, q_{13}, q_{32}, q_{33} $ equal to zero (because the constraint $ P(y_1,x_0|z_1) = q_{02} + q_{03} + q_{22} + q_{23}$, the terms $q_{02}, q_{03}, q_{22},q_{23}$ do not affect the value of $P(y_1|{do}(x_0))$). 

Similarly, for the maximum of $P(y_1|{do}(x_1))$, without loss of generality, assume $P(y_0,x_1|z_1)>P(y_0,x_1|z_0)$. To maximize $ P(y_1|{do}(x_1)) = \sum_{i=0}^3 q_{i1} + q_{i3} $, let $q_{01}+q_{21}=P(y_0,x_0|z_1)$, and $q_{03}+q_{23}=P(y_1,x_0|z_1)$ (due to the constraint $P(y_1,x_1|z_1) = q_{11} + q_{13} + q_{31} + q_{33}$ ,the terms $q_{11},q_{13},q_{31},q_{33}$ do not affect the value of $P(y_1|{do}(x_1))$). This is equivalent to making $q_{00}, q_{20}, q_{02}, q_{22}$ equal to zero. 

Combining the two derivations above, any SCM with response variables such that $q_{00}, q_{20}, q_{02}, q_{22}, q_{12}, q_{13}, q_{32}, q_{33} $ equals to zero would attains the maximum of $P(y_1|{do}(x_1))$ and minimum of $ P(y_1|{do}(x_0))$.
\end{proof}

\subsection{Properties of the Non-ID region/inner bound}

\section*{Preliminaries}

\noindent
\begin{tabular}{cl}
$r_n := c/\sqrt{n}$, \quad $c := \sqrt{\tfrac{1}{2}\ln(2m/\alpha)}$ & confidence radius at sample size $n$ \\
$I_P = [\ell(P), u(P)]$ & ATE bound set for distribution $P$ \\
$\Lambda$ & Lipschitz constant of $\ell, u$ in $P$ \\
$w^\ast := |I_{P^\ast}| > 0$ & true non-ID width (fixed, population quantity) \\
$\mathcal{C}_n := B(\hat P_n, r_n)$ & confidence set at sample size $n$ \\
$J_n := \bigcap_{P \in \mathcal{C}_n} I_P$ & reported non-ID region \\
$K_{n} := \big[\ell(P^\ast) + 2\Lambda r_{n},\; u(P^\ast) - 2\Lambda r_{n}\big]$ & guaranteed core at sample size $n$ \\
$\kappa := \dfrac{4\Lambda r_n}{w^\ast} > 1$ & resolution-excess factor
\end{tabular}

\medskip
\noindent
On the coverage event $\{P^\ast \in \mathcal{C}_n\}$, one has $K_n \subseteq J_n \subseteq I_{P^\ast}$, hence taking widths,
\begin{equation}
w^\ast - 4\Lambda r_n \;\le\; \omega_n \;\le\; w^\ast, \tag{6}
\end{equation}
where $\omega_n := |J_n|$ (signed, so $J_n = \emptyset \iff \omega_n < 0$).

\bigskip

\begin{theorem}[Resolving batch size]
Suppose at sample size $n$ the reported non-ID region $J_n$ is empty, while $w^\ast > 0$. Then, on the coverage event $\{P^\ast \in \mathcal{C}_n\} \cap \{P^\ast \in \mathcal{C}_{n'}\}$, of probability at least $1-2\alpha$, collecting
\[
k \;>\; k^\ast \;:=\; \frac{8\Lambda^2 \ln(2m/\alpha)}{(w^\ast)^2} - n
\qquad\Longleftrightarrow\qquad
k \;>\; (\kappa^2 - 1)\,n
\]
additional samples guarantees
\[
J_{n+k} \;\supseteq\; K_{n+k} \;\neq\; \emptyset.
\]
Equivalently, in terms of total sample size $n' = n+k$, this is $n' > N^\ast := \left\lceil \dfrac{8\Lambda^2\ln(2m/\alpha)}{(w^\ast)^2} \right\rceil$.
\end{theorem}

\begin{proof}
It suffices to show $K_{n'} \neq \emptyset$, since $K_{n'} \subseteq J_{n'}$ on the coverage event.

\smallskip
\noindent \textbf{Step 1 (nonemptiness of $K_{n'}$).}
$K_{n'}$ is an interval whose width is
\[
|K_{n'}| = \big(u(P^\ast) - 2\Lambda r_{n'}\big) - \big(\ell(P^\ast) + 2\Lambda r_{n'}\big) = w^\ast - 4\Lambda r_{n'}.
\]
An interval is nonempty iff its width is positive, so
\[
K_{n'} \neq \emptyset \iff w^\ast - 4\Lambda r_{n'} > 0 \iff w^\ast > 4\Lambda r_{n'}.
\]

\smallskip
\noindent \textbf{Step 2 (solving for $k$).}
Substituting $r_{n'} = c/\sqrt{n'}$,
\[
w^\ast > \frac{4\Lambda c}{\sqrt{n'}}
\;\Longleftrightarrow\;
\sqrt{n'} > \frac{4\Lambda c}{w^\ast}
\;\Longleftrightarrow\;
n' > \frac{16\Lambda^2 c^2}{(w^\ast)^2}.
\]
Both sides were positive throughout, so squaring preserves the inequality. Substituting $c^2 = \tfrac{1}{2}\ln(2m/\alpha)$ gives
\[
n' > \frac{8\Lambda^2 \ln(2m/\alpha)}{(w^\ast)^2} =: N^\ast.
\]
Writing $n' = n + k$,
\[
n + k > N^\ast
\;\Longleftrightarrow\;
k > \underbrace{\frac{8\Lambda^2 \ln(2m/\alpha)}{(w^\ast)^2} - n}_{=\,k^\ast}.
\]

\smallskip
\noindent \textbf{Step 3 (equivalence with $\kappa$).}
Since $c = r_n \sqrt{n}$, so $c^2 = r_n^2 n$,
\[
N^\ast = \frac{16\Lambda^2 c^2}{(w^\ast)^2} = \frac{16\Lambda^2 r_n^2 n}{(w^\ast)^2} = \left(\frac{4\Lambda r_n}{w^\ast}\right)^{\!2} n = \kappa^2 n.
\]
Hence $k^\ast = N^\ast - n = \kappa^2 n - n = (\kappa^2 - 1)n$, i.e. $k > k^\ast \iff k > (\kappa^2-1)n$.

\smallskip
Combining Steps 1--3: $k > (\kappa^2-1)n \implies n' > N^\ast \implies w^\ast > 4\Lambda r_{n'} \implies K_{n'} \neq \emptyset \implies J_{n+k} \supseteq K_{n+k} \neq \emptyset$.
\end{proof}

\clearpage

\section{Algorithm Details and Complexity Analysis}
\label{sec:comp-anal}
\label{sec:alg_detail}

\subsection{From Ambiguous to Unambiguous Decision for Single Action}

We can reach the final decision through a sequence of three moves.
Let $\gamma =  P(y)$.
\textit{1. Return:} if $ \gamma < \Lox$, {we decide the action is beneficial in population level}. 
\textit{2. Observe}: if $\overline{L_{x}}< \gamma<\underline{U_x}$, then we observe additional variables for example instrument variables.
Observing additional variables gives us a small causal effect set $\mathcal{I}'_P$, for each $P\in \CR$. Thus, for fixed $\CR$, we get a narrower inner region, $[\Lp_{x}, \Up_{x}] =  \bigcap_{P\in\CR}\mathcal{I}_P'$ 
and a narrower/wider outer band, $[\Lo_x,\Lp_x]\cup[\Up_x,\U_x]=\bigcup_{P\in\CR} \mathcal{I}_P'\setminus \bigcap_{P\in\CR} \mathcal{I}_P'$.
Also, the average non-id bound width will reduce since  $1/|\CR|  \sum_{P\in \CR}|\mathcal{I}'_P| \leq 1/|\CR|  \sum_{P\in \CR}|\mathcal{I}_P|$.

\textit{3. Collect}: In all other scenarios, we are unsure about the source of uncertainty (low sample size or non-identifiability) and we collect more samples, with the assumption that obtaining samples with additional variables is more challenging compared to collecting samples of the same set of variables. More samples provides us a smaller confidence set $\CR'(\hat{P})$. 
Thus, according to Theorem~\ref{thm:intersect} and ~\ref{thm:intersect-mult}, we get a wider inner region, $[\Lp_{x}, \Up_{x}] =  \bigcap_{P\in\CR'}\mathcal{I}_P$ 
and according to Corollary~\ref{outer-region} and ~\ref{outer-region-mult}, a narrower outer band, $[\Lo_x,\Lp_x]\cup[\Up_x,\U_x]=\bigcup_{P\in\CR'} \mathcal{I}_P\setminus \bigcap_{P\in\CR'} \mathcal{I}_P$.
The \textit{collect} and \textit{observe} moves can continue until we reach the \textit{return} move and decide the best action.

\subsection{Effect of Sample Size on an Individual Action Interval}
Suppose we are given $N$ samples from the true distribution $P(.)$, resulting in an empirical distribution $\hat{P}$. From the samples, we can construct confidence regions $\CR_{P(.)}$. Any distribution $P_{\epsilon}\in \CR_{P(.)}$ with $d(P_{\epsilon}, \hat{P})<\epsilon$ can be the true distribution. We maximize the causal effect while keeping the empirical distribution in this region. Suppose all the empirical distributions $P_{\epsilon} \in \CR$ provide us a set of max causal effects $M_{P(y|\Do(x))}$.

Now, suppose we have a larger confidence region $\CR^L \supseteq \CR$, e.g., obtained from the same empirical distribution with fewer samples or a higher confidence level. $\CR^L$ allows more distributions as potential candidates to be the true distribution.
As a result, we will have a new set of max causal effects $M'_{P(y|\Do(x))}$ for the distributions $P_{\epsilon}\in \CR^L \setminus \CR$.
Now, $\U = \max(M_{P(y|\Do(x))},M'_{P(y|\Do(x))}) \geq \max(M_{P(y|\Do(x))})$
and $\Up = \min(M_{P(y|\Do(x))},M'_{P(y|\Do(x))}) \leq \min(M_{P(y|\Do(x))})$. Thus, for a larger confidence region, the bound $[\Up, \U]$ extends.
Similarly, the bound $[\Lo, \Lp]$ also extends.
As the above two bounds extend, the bound $[\Lp, \Up]$ shrinks.

Now, suppose we have a smaller confidence region $\CR^S \subseteq \CR$, allowing fewer distributions as potential candidates to be the true distribution.
As a result, we will have a new set of max causal effects $M'_{P(y|\Do(x))}$ for the distributions $P_{\epsilon}\in \CR^S$.
Now, $\U = \max(M'_{P(y|\Do(x))}) \leq \max(M_{P(y|\Do(x))})$
and $\Up = \min(M'_{P(y|\Do(x))}) \geq \min(M_{P(y|\Do(x))})$. Thus, for a smaller confidence region, the bound $[\Up, \U]$ shrinks.
Similarly, the bound $[\Lo, \Lp]$ also shrinks.
As the above two bounds shrink, the bound $[\Lp, \Up]$ extends.

The above comparisons are exact when the confidence regions are nested, e.g., for the same empirical distribution with different confidence levels (Appendix~\ref{sec:scm1_detail}). When additional samples are collected, the empirical distribution itself changes, so the new confidence region is not guaranteed to be nested in the previous one. However, since the confidence radius shrinks at rate $O(1/\sqrt{n})$, the confidence region typically shrinks and concentrates around $P^\star$, and the above behavior typically holds, while the inner region remains contained in $\mathcal{I}_{P^\star}$ with probability at least $1-\alpha$ (Proposition~\ref{prop:irred-samp}).

The consequence of sample size can also be visualized as imagining $\max p(y|\Do(x))$ as an estimand. With finite samples, there is a confidence interval below and above it, which is $\min\max p(y|\Do(x))$ and $\max\max p(y|\Do(x))$. With fewer samples, this confidence interval typically gets wider.
So the $\min\max p(y|\Do(x))$ will typically be smaller.
Similarly, the $\max\min p(y|\Do(x))$ will typically be larger.

\textbf{Hyper parameters:} 
 $(\alpha)$: In Alg~\ref{alg:dcm}, determines how fast the model’s implicit distribution should return to the $\epsilon$-ball around the empirical distribution (i.e., controls the soft constraints).
 $(\epsilon_s)$: In Alg~\ref{alg:decide}, determines the size of the search space around the empirical distribution. In general, smaller $\epsilon_s$ values provide tighter estimates but require a larger number of candidate balls within the search space.

 \begin{figure}[t!]
     \begin{minipage}[t]{1\textwidth}
  \begin{minipage}[t]{1\textwidth}
       \footnotesize
\begin{algorithm}[H]
     \footnotesize
      \caption{$\mathrm{RelaxedDCM}$ $\mathrm{Training}$ $\mathrm{Model}$}
      \label{alg:dcm}
      \begin{algorithmic}[1]
         \State \textbf{Input:} 
         Data distribution $P_{\epsilon}(\mathbf{V})$,
         Variables $\mathbf V$, graph $\mathcal G,$ action $x$, smaller interval $\epsilon_s$      
\For{$j\in [0, 1]$ \Comment{ 0: minimization, 1: maximization}}
\State Initialize dual parameter $\{\alpha_i\}_{V_i\in \mathbf{V}}$ and ate weight $\lambda$.
\State Initialize $\mathrm{DCM}$  parameters 
$\Theta \gets \{\theta_i\}_{V_i \in \mathbf V}$ 
\label{alg2:ncm-init}
\For{$epoch=1\to max\_epoch$}
\State $d_i = dist\left( P_{\epsilon}(v_i \mid v_{\pi^{(i-1)}}) ,  P_{\Theta_i}(v_i \mid v_{\pi^{(i-1)}}) \right), \forall_i : V_i \in \mathbf{V}$
\State $L \gets \frac{1}{n} \sum_{k=1}^n
\sum_{V_i\in \mathbf{V}}
\alpha_i  \left[
d_i
- \epsilon_s \right]$
\State $\mathcal{L} \gets L $ +
$(-1)^j*\lambda *
ATE^{(\Theta)}(x)$ \Comment{ ATE with respect to the best among remaining actions.}
\label{alg2:min-eff}
\State $\alpha_{i} \gets \max\left(0,\ \alpha_i + \mathit{lr} \cdot (d_i - \epsilon_s)\right), \forall _i: V_i\in \mathbf{V}$ 
\State $\Theta\gets \Theta - \eta\,\nabla\mathcal L_{}$
          \EndFor
          \State $bound[j] = ATE^{(\Theta)}(x)$ 
        \EndFor
    \State \textbf{Return} $\{U_x, L_x\}= \{bound[1], bound[0]\}$
      \end{algorithmic}
    \end{algorithm}
  \end{minipage}%
  \end{minipage}
\end{figure}

\begin{algorithm}[t!]
     \footnotesize
    \caption{$\mathrm{Construct}$ $\epsilon$-net}
      \label{alg:enet}
      \begin{algorithmic}[1]
\State \textbf{Input:} data $\{\mathbf v_k\}_{k=1}^n$, Variable set $\mathbf V$, graph $\mathcal G$, small interval $\epsilon_s$  
\State Factorize $\hat{P}(\mbf{v})$ in $\Pi_{v\in \mbf{v}} \hat{P}(v_i|v^{\pi_i-1})$
    \For{ each $V_i \in \mathbf{V}$}
    \State $\epsilon = \sqrt{\frac{\ln(2/\alpha)}{2 n}}$ where $n= count(D[v_{\pi}^{(i-1)}])$
    \State $C_{i}$ =$\ceil{\epsilon/\epsilon_s}$ equidistance points in $\hat{P}_n(v_i=1|v_{\pi}^{(i-1)}) \pm \epsilon$
    \EndFor
    \State Sample $m$ joints s.t.$\{(p^{(j)}_1,..,p^{(j)}_{|\mathbf{V}|})  \}^{m}_{j=1}  \subseteq  C_1 \times C_2 \times...\times C_{|\mathbf{V}|}$
    \State \textbf{Return} $m$ candidate joint distributions
    $\{ P^{(j)}  \}^{m}_{j=1}$.
  \end{algorithmic}
    \end{algorithm}

\subsection{Algorithm~3: Construct $\epsilon$-net.}
\label{derivation}

There are various ways to construct a confidence set, extending on the classical confidence intervals. We use a simple set defined as the cartesian product of intervals obtained via concentration inequalities in each coordinate of the joint distribution, given as follows: $\mathcal{S}\coloneqq \prod_{v\in\mathcal{V}} [l_v,u_v],$
where $v$ is a configuration of the joint variable set  and $l_v, u_v$ are the lower and upper bounds that contain true $p(v)$ with probability $\alpha/m$, where $m$ is the total number of configurations in the observed distributions. Thus, $p\in S$ with probability $\alpha$ through a simple union bound argument. 

\textbf{Calculate Confidence Intervals}
We calculate the confidence interval for any conditional distribution $P(v_i|v^{\pi_i-1})$ as below:

First, we need to obtain the empirical estimators $\hat{P}(v_i=1|v^{\pi_i-1})$ for the binary variables from the training data.
$ \hat{P}_n(v_i=1 \mid v^{\pi_i-1}) = \frac{1}{m} \sum_{k=1}^{m} \mathbbm{1}(v_i^k(v^{\pi_i-1}) = 1) $

Now, let $n = count(D[v^{\pi_i-1}])$. Then according to Hoeffding's inequality:
\begin{equation}
    \begin{split}
        & Pr(|\hat{P} - P| \geq \epsilon) \leq 2 e^{-2n\epsilon^2}\\
        &\alpha= 2 e^{-2n\epsilon^2}\\
        &\alpha/2= e^{-2n\epsilon^2}\\
        &ln (2/\alpha)= 2n\epsilon^2\\
        &\frac{ln (2/\alpha)}{2n}=\epsilon^2\\
        &\epsilon = \sqrt{\frac{ln (2/\alpha)}{2n}}\\
    \end{split}
\end{equation}

Thus, the confidence interval for our empirical estimators are:

$ \CR_{P(v_i=1|v^{\pi_i-1})} := 
\left( 
\hat{P}_n(v_i=1|v^{\pi_i-1}) - \sqrt{\frac{\ln(2/\alpha)}{2n}}, \ 
\hat{P}_n(v_i=1|v^{\pi_i-1}) + \sqrt{\frac{\ln(2/\alpha)}{2n}}
\right) $

$ P(v_i=1|v^{\pi_i-1}) \in \CR_{P(v_i=1|v^{\pi_i-1})} \quad \text{with probability greater than } 1 - \alpha $

Note that the number of confidence intervals increases with the support size and the number of variables. If we consider $\ceil{\epsilon/\epsilon_s}$ centroids in each interval, the number of all possible joint distributions will be extremely high. To deal with such a scenario, we uniformly pick a centroid from each conditional distribution interval and form a valid joint distribution $P$. For a binary two variable case $X,Y$, we can split the CI of $P(X=1)$ into $n+1$ smaller intervals by placing $n$ centroids. We pick uniformly a value $p$ from these $n$ centroids having $P(X=0)=1-p$. Doing the same for $P(Y=1|X=0)$ and $P(Y=1|X=1)$ will allow us to construct $P(X,Y)$.
However, since we are sampling the joint distribution only a finite number of times, it is possible that we miss the true distribution. We have to accept a small error in the constructed bounds $\four$ as well. We leave the development of better search strategies for future work. While exploring the confidence interval, if a candidate distribution becomes inconsistent with the assumed instrument graph, we reject it based on conditions proposed in \citep{pearl2013testability}. Finally, our algorithm can deal with positivity violation of any distribution by considering the whole region $[0,1]$ as its confidence interval. 

\textbf{Complexity:}
Let $T$ denote the number of epochs, $B$ denote the batch size, $|D|$ denote the dataset size, $|V|$ denote the number of variables, $\mathcal{X}$ denote the cardinality of each variable. 
We construct confidence intervals from dataset $D$ for each of the $|V|$ variables, which costs $O(|V||D|)$. Sampling $M$ joint distributions from these intervals adds $O(M)$ cost. Thus,
\[
A_3 = O(|V||D| + M).
\]

\subsection{Algorithm~2: RelaxedDCM Training Model.}
For the Deep Causal Models
defined in \Cref{def:scm}, we have the following theorem. 
\begin{theorem}
\citep{kocaoglu2018causalgan,xia2021causal,rahman2024modular}
	\label{th:identifiability}
	Consider any SCM $S=(\mathbf{V}, \mathcal{N}, \mathcal{U}, \mathcal{F}, P(.) )$.  A DCM $\G=\{f_{1},...,f_{n}\}$ for $G$ entails the same identifiable interventional distributions as the SCM $S$ if it entails the same observational distribution.  
	\end{theorem}

\textbf{Complexity:}
Let $T$ denote the number of epochs, $B$ denote the batch size, $|V|$ denote the number of variables, $\mathcal{X}$ denote the cardinality of each variable, $c$ denote the forward-backward pass per sample cost, and $m$ denote the number of Monte Carlo samples. We use one neural network for each variable.
The joint distribution of $|V|$ variables has $\mathcal{X}^{|V|}$ configurations. In each epoch, we process $B$ samples of each configuration, with $m$ Monte Carlo samples, evaluate $|V|$ neural networks, and repeat the procedure for both minimization and maximization.

The total cost is
\[
A_2
=
T \cdot 2 \cdot \mathcal{X}^{|V|} \cdot B \cdot m \cdot |V| \cdot c
=
2T\mathcal{X}^{|V|}B|V|mc
=
O\!\left(T\mathcal{X}^{|V|}B|V|mc\right).
\]

\subsection{Algorithm~1: Explore $\epsilon$-ball.}

\textbf{Complexity:}
Let $T$ denote the number of epochs, $B$ denote the batch size, $|D|$ denote the dataset size, $|V|$ denote the number of variables, $\mathcal{X}$ denote the cardinality of each variable, $c$ denote the forward-backward pass per sample cost, and $m$ denote the number of Monte Carlo samples. We use one neural network for each variable.

After constructing the $\epsilon$-net using Algorithm~3, we execute RelaxedDCM using Algorithm~2 for each of the $M$ candidate joint distributions. Thus,
\[
A_1
=
O(A_3 + M \cdot A_2)
=
O\!\left(
|V||D| + M + M \cdot T\mathcal{X}^{|V|}B|V|mc
\right).
\]

The overall complexity is dominated by Algorithm~2 and scales as $\mathcal{X}^{|V|}$ due to operations over the joint space. Therefore, the complexity is exponential in the number of variables $|V|$ and polynomial in the domain size $\mathcal{X}$ with degree $|V|$.

\subsection{Graph-Induced Constraints on the Confidence Set}
\label{sec:graph-constraints}

The confidence set $\mathcal{C}(D)$ is constructed from the data and is not, in general, contained in $\mathcal{M}(\mathcal{G})$, the set of observational distributions entailed by some SCM compatible with $\mathcal{G}$. Even when $P^\star \in \mathcal{M}(\mathcal{G})$, distributions in $\mathcal{C}(D)$ that are far from the empirical distribution may violate the constraints implied by $\mathcal{G}$. For such a distribution $P \notin \mathcal{M}(\mathcal{G})$, no compatible SCM entails $P$, so $\mathcal{I}_{P}(x) = \emptyset$. Consequently, whenever $\mathcal{C}(D) \not\subseteq \mathcal{M}(\mathcal{G})$, the inner region $\bigcap_{P \in \mathcal{C}(D)} \mathcal{I}_{P}(x)$ is empty in exact arithmetic, while in practice the neural optimization returns non-empty approximate intervals for such $P$, which can distort both the inner and outer regions. Restricting the confidence set to $\mathcal{C}(D) \cap \mathcal{M}(\mathcal{G})$ preserves coverage when $\mathcal{G}$ is correctly specified and removes such problems.

Whether this restriction is needed depends on the graph. For graphs such as the bow graph, $\mathcal{M}(\mathcal{G})$ imposes no constraints on the observational distribution. For causally sufficient graphs, the constraints are enforced by construction, since we build confidence intervals on local conditionals. With latent variables, factorizing over c-components captures part of this structure but does not enforce inequality constraints such as the instrumental inequalities. For the IV graph, we therefore check the instrumental inequalities for the distributions explored in the $\epsilon$-net and discard those that violate them. 

For the binary IV graph, these inequalities~\citep{pearl1995testability} are linear in $p(x, y \mid z)$ and characterize the observational class exactly~\citep{balke1997bounds, bonet2001instrumentality}, so the restricted confidence set is a polytope over which an $\epsilon$-net can be constructed directly. Constructing confidence sets within the observational class of a general graph with latent variables remains an open problem.

\subsection{Computational Complexity:}
When the training dataset contains small number of samples for variables with large number of states, the search space of epsilon-net is large and therefore more iterations are needed to find the min-max and max-min. In that case, the size of epsilon-net can be adjusted with a trade-off between accuracy and efficiency.
As the number of samples increases, the search space of epsilon-net will become smaller and therefore more efficient when searching for min-max and max-min. 
Finally, when the empirical distribution converges to the true distribution for a large dataset, and the confidence intervals shrink accordingly, our method recovers the true bounds for non-identifiable queries, similar to existing approaches. Our main contribution lies in handling the uncertainty region where we can only estimate the empirical distributions with finite samples. We can utilize and separate the sources of uncertainty to make optimal decisions or develop data collection strategies to improve the results.

We pick each distribution $P_{\epsilon_s}$ and train two models for at least 2000 epochs to maximize and minimize the ATE respectively while keeping the model learned distribution within $\epsilon_s$ distance of $P_{\epsilon_s}$.
We learn $k$ ($\sim 200$) such distributions and achieve $k$ many $[\min, \max,]$ bounds. We can then construct the $[\Late, \Lpate, \Upate, \Uate]$ following Equation~\ref{eq:4quant-mult}.

All experiments were conducted on workstations equipped with two NVIDIA GeForce RTX $4090$ GPUs, each with $24$ GB of memory. We executed the algorithm using $15$-$20$ parallel threads, where each thread explores one candidate joint distribution from the $\epsilon$-ball at a time. The runtime depends on the problem instance, computational resources, and hyperparameter choices. For a dataset with $18$ binary variables, we train for $500$-$1000$ epochs with a batch size of $256$, processing each meaningful joint configuration. A single run takes approximately $20$ minutes to $1$ hour, and evaluating $100$-$200$ candidate distributions takes approximately $6$-$7$ hours under parallel execution.

\section{Experiments}
\label{sec:appex-experiment}

\subsection{Causal Graphs}
The following causal graphs are used in our experiments.

\begin{figure}[H]
\centering
\hspace{-10mm}
\begin{subfigure}{0.31\linewidth}
\begin{minipage}[t]{1\textwidth}
\centering
\begin{tikzpicture}[scale=0.5, transform shape, node distance=1cm, every node/.style={draw, circle, minimum size=1cm}, ->, >=Stealth]
\node[align=left] (X) {X};
\node[right=of X] (Y) {Y};
\node[above=of X, fill=lightgray] (U) {U};
\draw (X) -- (Y);
\draw (U) -- (X);
\draw (U) -- (Y);
\end{tikzpicture}
\end{minipage}
\caption{Bow graph.}
\label{fig:bow-graph}
\end{subfigure}
\begin{subfigure}{0.31\linewidth}
\begin{minipage}[t]{1\textwidth}
\centering
\begin{tikzpicture}[scale=0.5, transform shape, node distance=1cm, every node/.style={draw, circle, minimum size=1cm}, ->, >=Stealth]
\node[align=left] (I) {I};
\node[right=of I] (X) {X};
\node[right=of X] (Y) {Y};
\node[above=of X, fill=lightgray] (U) {U};
\draw (I) -- (X);
\draw (X) -- (Y);
\draw (U) -- (X);
\draw (U) -- (Y);
\end{tikzpicture}
\end{minipage}
\caption{Instrument variable (IV) graph.}
\label{fig:iv-graph}
\end{subfigure}
\begin{subfigure}{0.31\linewidth}
\begin{minipage}[t]{1\textwidth}
\centering
\begin{tikzpicture}[scale=0.5, transform shape, node distance=1cm, every node/.style={draw, circle, minimum size=1cm}, ->, >=Stealth]
\node[align=left] (X) {X};
\node[right=of X] (Y) {Y};
\node[above=of X] (U1) {$Z_1$};
\node[right=of U1] (U2) {$Z_2$};
\draw (U1) -- (X);
\draw (U2) -- (X);
\draw (U1) -- (Y);
\draw (U2) -- (Y);
\draw (X) -- (Y);
\draw (U) -- (X);
\draw (U) -- (Y);
\end{tikzpicture}
\end{minipage}
\caption{Two observed confounder graph.}
\label{fig:two-obs-cnf}
\end{subfigure}
\caption{Causal graphs used in experiments}
\label{fig:all-graphs}
\end{figure}

\subsection{Experiment Detail for \Cref{sec:scm1}}
\label{sec:scm1_detail}
\textbf{Setup:} In the first experiments, we consider the causal graph shown in Figure~\ref{fig:iv-graph} where we have 2 instrument variables ($I_1, I_2)$, one unobserved confounder ($U$), a treatment ($X$) and an outcome variable ($Y$). All of them are considered binary. 
The ground truth ATE of this causal model is 0.236, i.e, $X_1$ is the best action.
To illustrate the gradual progression of our algorithm we consider four setups: setup 1) observed variables: $\{X,Y\}$ are sampled from $P(x,y)$ as dataset $D[X,Y]$ of size=1000; setup 2) observed: $\{I_1, X,Y\}$, are sampled from $P(i, x,y)$ as dataset $D[I_1, X,Y]$ of size 1000; setup 3) observed $\{I_1, I_2, X,Y\}$ are sampled from $P(\mathbf{V})$ as  $D[\mathbf{V}]$ of size 1000; setup 4: same as setup 3 but with 3000 samples. The goal is to choose one of the two actions ${do(}X=0)$ and ${do(}X=1)$ as the best action.

\textbf{Baselines:}
Given the dataset, we estimate the corresponding joint distribution and use closed-form expression such as  Tian and Pearl bound~\citep{tian2000probabilities} 
for bow graph with no instruments and IV bound when we have one or more instrument variables. We also use NCM~\citep{xia2021causal} and autobound~\citep{duarte2024automated} to obtain bounds for the input distribution.

\textbf{Confidence levels:}
We evaluate the behavior of our algorithm for different confidence levels. We plug in the closed-form bound estimator and execute Algorithms~1 and~3 with $10k$ samples but different confidence levels. We observe that for $[70\%, 90\%]$ confidence levels, the outer region is always greater than zero, while for $95\%$ it is not. Thus, for a fixed sample size, we have to reduce our confidence level to find the best action.
For $10k$ samples, the estimated interval $[\mathrm{minmin}, \mathrm{maxmax}]$ is $[-0.02, 0.48]$ at $95\%$ confidence level, $[0.02, 0.48]$ at $90\%$ confidence level, and $[0.02, 0.45]$ at $70\%$ confidence level.

\textbf{Sample size:}
One metric to evaluate the cost of our decision making might be the \textbf{sample size}. We extend \Cref{sec:scm1} by performing additional experiments for sample sizes $8k$ and $16k$ and show how the bound changes.
For $3k$ samples, the outer bound is $[-0.083, 0.52]$, and the resulting strategy is to collect more samples. For $8k$ samples, the outer bound is $[0.0042, 0.43]$, and the resulting strategy is to return $\mathrm{do}(X=1)$. For $16k$ samples, the outer bound is $[0.04, 0.39]$, and the resulting strategy is to return $\mathrm{do}(X=1)$.
Although our algorithm requires a high number of samples, it follows a principled way to propagate uncertainty using causal inference machinery and offers a confident decision unlike existing baselines.

\subsection{Experiment Detail for \Cref{sec:scm2}}
\textbf{Setup:}
In this experiment, we consider the causal graph shown in Figure~\ref{fig:two-obs-cnf} where we have 2 observed confounders ($Z_1, Z_2)$, a treatment ($X$) and an outcome variable ($Y$). All of them are considered binary. In this setup, we do not have any unobserved confounder and thus no uncertainty from non-identifiability of ATE or $P(Y|\Do(X))$. 
The ground truth ATE=0.06.
Here we consider two setups with the same graph: setup 1) observed variables: $\{Z_1, Z_2, X,Y\}$ are sampled from $P(\mathbf{V})$ as dataset $D[Z_1, Z_2, X,Y]$ of size=1000; setup 2) same as setup 1 but with 3000 samples. The goal is to choose one of the two actions ${do(}X=0)$ and ${do(}X=1)$ as the best action.

\textbf{Baselines:}
Since there exists no unobserved confounders, we can compare our algorithm output with existing baselines that assume unconfoundedness. We use the dowhy package~\citep{dowhy,JMLR:v25:22-1258} to execute 6 baselines such as: propensity score matching, doubly robust estimator etc., on our data. These methods provides a point-wise estimation. We also execute the NCM approach which provides a bound for the ATE estimation.

\subsection{Scaling Experiment}
\label{fig:scale}
 The overall complexity of the naive implementation of our algorithm is exponential  in the number of variables $|V|$ and polynomial in the domain size $\mathcal{X}$ with degree $|V|$. Please see Section~\ref{sec:comp-anal} for complexity details. 
 The computation ($\mathcal{X}^{|\mathbf{V}|}$) is required since 
 the method executes Algorithm~\ref{alg:dcm} for each candidate distribution by exploring the $\epsilon$-ball. 
 To scale our algorithm for large graphs, we utilize two practical solutions. 1. \textbf{Reduce the number of confidence intervals:} The joint distribution factorizes according to the causal graph 
    $P(V)=\prod_{V_i \in V} P(V_i \mid pa(V_i))$. We match the factorized conditional distributions instead of the full joint. Similar factorization with c-components can be leveraged when latents are present.
2. \textbf{Ignore low-probability joint combinations:} For large graphs with large-cardinality variables, many combinations are assigned a very low probability. A practical solution is to identify these very low-probability combinations and redistribute their probability mass to the rest of the states and fit that distribution. This reduces the slicing of the alpha value to a large number of combinations that do not meaningfully change the causal effect since the resultant perturbation of the observations is small. Accordingly, for an $18$-node graph, we end up with just $\sim 400$ meaningful combinations.

\subsection{Experiment Details for Parent Labor Dataset}
\label{appex:parent-labr}

~\citet{angrist1996children} focus on the causal link running from fertility to the work effort of both men and women. 
According to~\citet{angrist1996children}, we consider an instrumental variable $I$: the sex of the first two children in families with two or more children. This variable can be justified as a valid instrument  since it represents the widely observed phenomenon of parental preferences for a mixed sibling-sex composition. Parents of same-sex children are significantly likely interested to have an additional child. Also, sex is randomly assigned, $I$ does not have any causal parent.
 Thus, it is a plausible instrument for further child-bearing among women with at least two children.
 We obtain the dataset from ~\cite{angrist_evans_2009_replication}.

Angrist and Evans~\cite{angrist1996children} estimate the causal effect of fertility on parents' labor supply using an instrumental variable strategy based on the sex composition of the first two children. Their treatment variable is whether a family has more than two children, and their main instrument is an indicator for whether the first two children are of the same sex. The identifying intuition is that parents with two same-sex children are more likely to have an additional child due to preferences for a mixed sibling-sex composition, while the sex composition of the first two births is plausibly as good as randomly assigned. Using two-stage least squares (2SLS), they first predict the probability of having more than two children using the same-sex instrument and then estimate the effect of the predicted fertility variation on labor-supply outcomes. In their main 1980 PUMS specification for women aged $21$ to $35$ with at least two children, the 2SLS estimates using the same-sex instrument show that having more than two children reduces the probability that a woman worked for pay by $12.1$ percentage points, annual weeks worked by $5.68$ weeks, weekly hours worked by $4.61$ hours, and labor income by approximately $\$1{,}960.5$. They also find little evidence that additional children reduce husbands' labor supply. Overall, their results indicate that additional childbearing causally reduces mothers' labor supply and earnings, while having little effect on fathers' labor-market outcomes.
These estimates are reported in Table~5 of Angrist and Evans~\cite{angrist1996children}.

\subsection{How Model Misspecification in the Causal Graph Affects the Final Results}
\label{sec:mis-specify}

We analyze the final result for a scenario where incorrect edges are present.

\paragraph{True graph.}
The true graph contains $I_1 \rightarrow X, I_2 \rightarrow X, X \rightarrow Y$, and an unobserved confounder $U$ between $X$ and $Y$.

\begin{figure}[H]
\centering
\hspace{-10mm}
\begin{subfigure}{0.31\linewidth}
\begin{minipage}[t]{1\textwidth}
\centering
\begin{tikzpicture}[
    scale=0.5,
    transform shape,
    node distance=1cm,
    every node/.style={draw, circle, minimum size=1cm},
    ->,
    >=Stealth
]
\node (X) {$X$};
\node[right=of X] (Y) {$Y$};
\node[above=of X, fill=lightgray] (U) {$U$};
\node[below left=of X] (I1) {$I_1$};
\node[below right=of X] (I2) {$I_2$};

\draw (X) -- (Y);
\draw (U) -- (X);
\draw (U) -- (Y);
\draw (I1) -- (X);
\draw (I2) -- (X);
\end{tikzpicture}
\end{minipage}
\caption{True graph.}
\label{fig:true-graph}
\end{subfigure}
\begin{subfigure}{0.31\linewidth}
\begin{minipage}[t]{1\textwidth}
\centering
\begin{tikzpicture}[
    scale=0.5,
    transform shape,
    node distance=1cm,
    every node/.style={draw, circle, minimum size=1cm},
    ->,
    >=Stealth
]
\node (X) {$X$};
\node[right=of X] (Y) {$Y$};
\node[above=of X, fill=lightgray] (U) {$U$};
\node[below left=of X] (I1) {$I_1$};
\node[below right=of X] (I2) {$I_2$};

\draw (X) -- (Y);
\draw (U) -- (X);
\draw (U) -- (Y);
\draw (I1) -- (X);
\draw (I2) -- (X);
\draw[bend left=15] (I2) to (Y);
\end{tikzpicture}
\end{minipage}
\caption{Wrong graph 1.}
\label{fig:wrong-graph-1}
\end{subfigure}
\begin{subfigure}{0.31\linewidth}
\begin{minipage}[t]{1\textwidth}
\centering
\begin{tikzpicture}[
    scale=0.5,
    transform shape,
    node distance=1cm,
    every node/.style={draw, circle, minimum size=1cm},
    ->,
    >=Stealth
]
\node (X) {$X$};
\node[right=of X] (Y) {$Y$};
\node[above=of X, fill=lightgray] (U) {$U$};
\node[below left=of X] (I1) {$I_1$};
\node[below right=of X] (I2) {$I_2$};

\draw (X) -- (Y);
\draw (U) -- (X);
\draw (U) -- (Y);
\draw (I1) -- (X);
\draw (I2) -- (X);
\draw[bend right=15] (I1) to (Y);
\draw[bend left=15] (I2) to (Y);
\end{tikzpicture}
\end{minipage}
\caption{Wrong graph 2.}
\label{fig:wrong-graph-2}
\end{subfigure}
\caption{Causal graphs used in the misspecification experiment. The true graph is a bow graph with two IVs. Wrong graph 1 adds the extra edge $I_2 \rightarrow Y$. Wrong graph 2 adds the extra edges $I_1 \rightarrow Y$ and $I_2 \rightarrow Y$.}
\label{fig:all-graphs}
\end{figure}
Wrong graph 1 has the extra edge $I_2 \rightarrow Y$. Wrong graph 2 has the extra edges $I_2 \rightarrow Y$ and $I_1 \rightarrow Y$.

Since any estimator can be incorporated in our framework, we pick a closed-form formula to calculate bounds for these graphs. We explore distributions in the $\epsilon$-ball and combine the bounds to obtain the following four quantities.

\begin{table}[H]
\centering
\caption{Effect of causal graph misspecification on the final causal effect bounds.}
\label{tab:misspecification-results}
\begin{tabular}{lccccc}
\toprule
\textbf{Setup} & \textbf{Instruments} & \textbf{minmin} & \textbf{maxmin} & \textbf{minmax} & \textbf{maxmax} \\
\midrule
True graph    & $I_1, I_2$ & $0.01$  & $0.11$  & $0.333$ & $0.41$ \\
Wrong graph 1 & $I_1$      & $-0.23$ & $-0.16$ & $0.52$  & $0.59$ \\
Wrong graph 2 & None       & $-0.35$ & $-0.32$ & $0.64$  & $0.67$ \\
\bottomrule
\end{tabular}
\end{table}

We observe that the additional incorrect edge loosens the causal effect bound compared to the correct graph, although it still encloses the correct bound.

\section{Reproducibility Statement}
Assumptions and complete proofs of our theoretical results are given in Appendix~\ref{thr-anl}. The algorithms, their complexity, and implementation details are described in Appendix~\ref{sec:alg_detail}. Causal graphs, data-generating processes, sample sizes, hyperparameters, and compute resources for all experiments are listed in Appendix~\ref{sec:appex-experiment}, with details for the real dataset in Appendix~\ref{appex:parent-labr}. The synthetic data can be regenerated from these descriptions, and the real dataset is public. Our code is available at \url{https://github.com/Musfiqshohan/UA-DCM}.

\end{document}